\documentclass[journal]{IEEEtran}
\usepackage{amsmath}
\usepackage{array}
\usepackage{cite}
\usepackage{amsfonts}
\usepackage{amssymb}
\usepackage{graphicx}
\usepackage{epstopdf}
\usepackage{float}
\usepackage{balance}
\usepackage[fancythm,fancybb]{jphmacros2e} 
\usepackage{amsfonts}
\usepackage{mathrsfs}
\usepackage{amsmath}
\usepackage{amsthm}
\usepackage[norefs,nocites]{refcheck}
\usepackage{bbold} 

\usepackage{setspace}  
\usepackage[bookmarks=true, hidelinks]{hyperref}
\usepackage{times}  
\usepackage[explicit]{titlesec}  
\usepackage[titles]{tocloft}  
\usepackage[bookmarks=true, hidelinks]{hyperref}
\usepackage{rotating}  
\usepackage[normalem]{ulem}  
\usepackage{comment}
\usepackage{mathtools}
\usepackage{comment}
\usepackage{bigints}
\usepackage{textcomp}
\usepackage{stackengine}
\usepackage{booktabs}
\usepackage{siunitx}
\usepackage[utf8]{inputenc}
\usepackage[american]{babel}
\usepackage{tikz}
\usetikzlibrary{shapes,arrows}
\usepackage{subfig}
\usepackage{color,soul}
\usepackage{dsfont}
\usepackage{bm}
\usepackage{url}
\usepackage{hyperref}

\usepackage{algorithmic,algorithm2e,float}
\usepackage{lipsum}
\usepackage{changepage}

\usepackage{lipsum}

\DeclareMathOperator{\tr}{tr}

\newcommand{\ts}{\textsuperscript}

\allowdisplaybreaks

\begin{document}
\IEEEoverridecommandlockouts

\title{Decentralized and Privacy-Preserving Learning of Approximate Stackelberg Solutions in Energy Trading Games with Demand Response Aggregators}
\author{Styliani~I.~Kampezidou$^1$,
        ~Justin~Romberg$^1$,
        ~Kyriakos~G.~Vamvoudakis$^2$,
        ~and~Dimitri~N.~Mavris$^2$
\thanks{$^1$S.I.~Kampezidou and J.~Romberg are with the School of Electrical and Computer Engineering, Georgia Institute of Technology, Atlanta, GA, $30332$, USA e-mail: (skampezidou@gatech.edu, jrom@ece.gatech.edu).}
\thanks{$^2$K.~G.~Vamvoudakis and D.~N.~Mavris are with the Daniel Guggenheim School of Aerospace Engineering, Georgia Institute of Technology, Atlanta, GA, $30332$, USA e-mail: (kyriakos@gatech.edu, dimitri.mavris@aerospace.gatech.edu)}
\thanks{This work was supported in part, by the Department of Energy under grant No. DE-EE$0008453$, by ONR Minerva under grant No. N$00014$-$18$-$1$-$2160$, by NSF under grant Nos. CAREER CPS-$1851588$, CPS-$2227185$, and by S\&AS $1849198$.}
}



\markboth{\tiny This work has been submitted to the IEEE for possible publication. Copyright may be transferred without notice, after which this version may no longer be accessible.}{}


\maketitle

\begin{abstract}
In this work, a novel Stackelberg game theoretic framework is proposed for trading energy bidirectionally between the demand-response (DR) aggregator and the prosumers. This formulation allows for flexible energy arbitrage and additional monetary rewards while ensuring that the prosumers' desired daily energy demand is met. Then, a scalable (linear with the number of prosumers), decentralized, privacy-preserving algorithm is proposed to find approximate equilibria with online sampling and learning of the prosumers' cumulative best response, which finds applications beyond this energy game. Moreover, cost bounds are provided on the quality of the approximate equilibrium solution. Finally, real data from the California day-ahead market and the UC Davis campus building energy demands are utilized to demonstrate the efficacy of the proposed framework and algorithm.
\end{abstract}


\section{Introduction}
\IEEEPARstart{T}{o} mitigate climate change concerns and keep the global temperature increase within the 1.5°C limit as compared to the pre-1900s baseline, renewable energy resources, and energy savings solutions have been deployed massively in the past decade. Several leading economies, including the United States, have designed electricity markets (soon to include about 76\% of global electricity generation \cite{net_zero_2050_global_energy_sector_policy_makers}) to offer recoupment mechanisms that will attract additional private investment in renewable generation, DR, and battery storage, and provide monetary rewards for selling clean energy or providing other grid ancillary services. Some of these wholesale electricity markets have recently allowed access to low-capacity participants, such as distributed energy resources and distributed load \cite{union2009directive} via an energy-type broker, the aggregator \cite{madrigal2014overview}. In fact, complex distributed load (mix of various generation, storage, and demand components \cite{he2013engage}) is expected to play a major role in the electricity generation decarbonization by 2030 and the net-zero economy goals by 2050 \cite{credible_pathways_EIA_2023}, together with battery storage, as they are capable of modifying demand based on price signals to improve the grid's operation during hours of peak demand, steep ramping, lack of voltage support, and other operational issues \cite{siano2014demand} rising in the pathway to the 2030 goal of 90\% electricity generation from renewables \cite{credible_pathways_EIA_2023}. Complex distributed load can additionally contribute via DR to energy price reduction, according to PJM \cite{PJM_market_rep_2019}, and CO$_2$ emissions reduction, since it was responsible for about 14\% of the total CO$_2$ emissions in California \cite{AB32_Act2006_GHG_emissions_data}, and 32\% in the entire United States \cite{EPA_Greenhouse_Emissions_1990_2020} in 2020. The framework proposed allows distributed market participation with improved rewards, scalability, and privacy, therefore being of great importance.



\emph{Related Work:} The advantage of the proposed framework, as compared to existing literature, lies in its bidirectional trading capability (prosumers can both sell and purchase energy, i.e., perform energy arbitrage that leads to higher rewards) and in the fact that all players compete in a game for profit, CO$_2$ emissions reduction and grid operational support, optimizing their own objectives, as opposed to having one player optimizing both sides (aggregator or prosumer), as traditionally done in the literature before \cite{ayon2017optimal, parvania2013optimal, nan2018optimal, agnetis2013load, vaya2014optimal, baringo2017stochastic, barhagh2019risk, correa2018robust}. Few works that have considered game settings, have not been able to consider bidirectional transactions (storage-like flexibility), desired demand constraints, scalability, privacy \cite{alshehri2019impact, gkatzikis2013role, maharjan2013dependable}, wholesale electricity market presence \cite{li2011optimal, gkatzikis2013role}, or approximation quality bounds \cite{gkatzikis2013role}. 



Traditionally, Stackelberg games (bilevel programs) are solved to equilibrium, by replacing the second-level optimization programs with their Karush-Kuhn-Tucker (KKT) equations and adding them as constraints to the upper-level optimization program. This reformulation was performed manually or automatically (GAMS) and resulted in a mathematical program with equilibrium constraints (MPEC) \cite{fochesato2022stackelberg, jin2023stackelberg, jiang2022stackelberg, daraeepour2015strategic}, a single-level optimization program. The solution of the single-level optimization program is not scalable in games with multiple followers (second-level programs) and multiple constraints per follower, preventing the large-scale integration of distributed load (thousands of prosumers) and hence associated market and environmental benefits. This paper proposes a decentralized, scalable game equilibrium-seeking algorithm for reaching a market bidding decision, that also preserves privacy, by eliminating the need for a player to access another player's objective, constraints (device models, desired demand, etc.), and parameters (generation prediction, demand footprint, personally identifiable (PI) behind the meter sensor data \cite{ryu2016development, sim2016estimation, sheikhi2016adaptive}, etc.), further supporting the Department of Energy's and National Institute for Standards and Technology's efforts towards a more private and secure grid (Energy Independence and Security Act \cite{Energy_ind_act}). The proposed algorithm may be applied to Stackelberg frameworks in cybersecurity \cite{kar2017trends}, airport and national security \cite{patrolLAX2010, abbas2017improving}, resource allocation \cite{sawyer2018flexible}, cloud computing \cite{zhou2022stackelberg}, autonomous driving \cite{niu2023stackelberg} etc.





\emph{Contributions:}
First, a non-zero-sum, non-cooperative Stackelberg game for the prosumer DR-aggregator game is proposed, allowing bidirectional market transactions (buying and selling energy) and ensuring that the total desired daily prosumer demand is met. Second, a decentralized, privacy-preserving, scalable learning algorithm is proven and deployed with real-market data for market bidding. A similar, but not scalable approach has been proposed before for other applications \cite{sinha2017evolutionary} with evolutionary algorithms. Third, theoretical bounds on the approximate $\epsilon$-Stackelberg solution are proved given the approximation and learning errors, missing from other approaches \cite{sinha2017evolutionary, mylvaganam2014approximate, sinha2017review}.

\emph{Nomenclature}:
\begin{equation*}
    \begin{array}{rr}
    T, K, J & \text{hours in day, basis dimension, samples (par.)} \\
    \bm{\lambda} {\in} \mathbb{R}^T & \text{market bid prices (par.)}\\
    \mathbf{x}_i^{\o{}} {\in} \mathbb{R}^T & \text{demands of prosumer $i$ before DR (par.)}\\
    \mathbf{x}_i {\in} \mathbb{R}^T & \text{demands of prosumer $i$ after DR (var.)}\\
    u_i {\in} \mathbb{R} & \text{inelasticity parameter of prosumer $i$ (par.)}\\
    a {\in} \mathbb{R} & \text{flexibility parameter (par.)}\\
    W_i {\in} \mathbb{R} & \text{daily demand desired by prosumer $i$ (par.)}\\
    Q_i {\in} \mathbb{R} & \text{daily demand traded by prosumer $i$ (par.)}\\
    g_{p_{i}}(\mathbf{x}_i) {\in} \mathbb{R} & \text{objective of prosumer $i$ (fun.)}\\
    \mathbf{p} {\in} \mathbb{R}^T & \text{prices offered by aggregator (var.)}\\
    \mathbf{d}^{\o{}} {\in} \mathbb{R}^T & \text{total prosumer demands before DR (var.)}\\
    \mathbf{d} {\in} \mathbb{R}^T & \text{total prosumer demands after DR (var.)}\\
    \mathbf{\Delta d} {\in} \mathbb{R}^T & \text{total prosumer demand difference (var.)}\\
    g_{a}(\mathbf{p}) {\in} \mathbb{R} & \text{objective of aggregator (fun.)}\\
    \mathbf{d}^{\star}(\mathbf{p}) {\in} \mathbb{R}^T & \text{map of aggregator and prosumers (fun.)}\\
    \bm{\phi}(\mathbf{p}) {\in} \mathbb{R}^K & \text{approximation basis for $\mathbf{d}^{\star}(\mathbf{p})$ (fun.)}\\
    \bm{\Theta_c}^{\star} {\in} \mathbb{R}^{K \times T} & \text{optimal value if $\mathbf{d}^{\star}(\mathbf{p}) {\in} \textrm{span}\{\bm{\phi}(\mathbf{p})\}$ (var.)}\\
    \bm{\Theta}^{\star} {\in} \mathbb{R}^{K \times T} & \text{optimal value if $\mathbf{d}^{\star}(\mathbf{p}) {\notin} \textrm{span}\{\bm{\phi}(\mathbf{p})\}$ (var.)}\\
    \bm{\hat{\Theta}} {\in} \mathbb{R}^{K \times T} & \text{estimator of $\bm{\Theta}^{\star}$ via RLS learning (var.)}\\
    \bm{\Pi}^j {\in} \mathbb{R}^{K \times K} & \text{RLS matrix update at sample $j$ (var.)}\\
    m {\in} \mathbb{R} & \text{forgetting factor in RLS (par.)}\\
    \mathbf{p}_r^j {\in} \mathbb{R}^T & \text{random aggregator prices for learning (par.)}\\
    (\mathbf{x}_i^{s}{,}\mathbf{p}^{s}) & \text{exact Stackelberg equil. if $\bm{\hat{\Theta}}{=}\bm{\Theta_c}^{\star}$ (var.)}\\
    (\mathbf{x}_i^{\epsilon s}{,}\mathbf{p}^{\epsilon s}) & \text{approxim. Stackelberg equil. if $\bm{\hat{\Theta}}{\neq}\bm{\Theta_c}^{\star}$ (var.)}\\
    & \text{due to approximation and learning errors}\\
    \bm{\epsilon}_m(\mathbf{p}) {\in} \mathbb{R}^T & \text{approximation error (fun.)}\\
    \epsilon_m^\mathrm{max} {\in} \mathbb{R}& \text{maximum value of approximation error (par.)}\\
    \bm{\epsilon}^j {\in} \mathbb{R}^{1 \times T} &  \text{learning error from RLS (fun.)}\\
    \beta_0, \beta_1 {\in} \mathbb{R} & \text{parameters in Persistence of Excitation (par.)}\\
    \tilde{g}_a(\mathbf{p};\bm{\hat{\Theta}}^J) & \text{aggr.'s objective with map approxim. (fun.)}\\
    L_a, L_{\phi} {\in} \mathbb{R} & \text{Lipschitz constants of $\tilde{g}_a(\mathbf{p};\bm{\hat{\Theta}}^J)$, $\bm{\phi}(\mathbf{p})$ (par.)}\\
    \delta {\in} \mathbb{R} & \text{Quadratic Growth condition parameter (par.)}\\
    \end{array}
\end{equation*}

The rest of this paper is structured as follows. Section \ref{sec:prob_form} proposes the game formulation between heterogeneous prosumers and the DR-aggregator. Section \ref{sec:onl_learn} proposes a decentralized, privacy-preserving learning algorithm of game equilibria that cannot be derived in closed-form \cite{kampezidou2021online}, and theoretical approximation quality bounds. Section \ref{sec:exp_res} gives experimental results on the California grid and market data and section \ref{sec:concl} summarizes conclusions and future work. An Appendix with proofs complements this paper.

\section{Stackelberg Game Formulation}\label{sec:prob_form}
Among common competition types, i.e., Bertrand, Cournot, Nash, and Stackelberg (subgame perfect of Nash), the first two assume multiple firms (aggregators) \cite{dastidar1995existence, cournot1927researches} which is not possible in these one-aggregator sign-up policy markets. It is known that the DR-aggregator cannot do worse by playing first in a Stackelberg game (sequential play), as compared to playing simultaneously with the prosumers in a Nash game \cite{simaan1973stackelberg}. Therefore, for simplicity of equilibria extraction and to avoid estimating the prosumers' strategies in Nash, a Stackelberg competition type is preferred.

The DR-aggregator plays first and offers a price vector $\mathbf{p} {\in} \mathbb{R}^T$ to all the prosumers who play second and simultaneously, without access to each other's actions and choose a demand vector $\mathbf{x}_i {\in} \mathbb{R}^T$. Consider a vector of bounded market prices $\bm{\lambda} {\in} \mathbb{R}^T$ with $\infty {>} \lambda^\mathrm{max} {\geq} \lambda_t {>}0$, $\forall t {\in} \{1,\cdots,T\}$ for each hour of the day-ahead market which is decided by a different mechanism \cite{Empowered_Aggregator}, not researched here. The Stackelberg game played decides the wholesale market bid ($\bm{\lambda}, \bm{\Delta} \mathbf{d}$).



\subsection{The prosumer's problem}

Assumptions \ref{ass:pros_sched_demand}, \ref{ass:pros_dependence} ensure a well-defined prosumer's problem \cite{grimsman2020stackelberg}. Demand $x_{it}^{\o{}}$ may be scheduled or predicted \cite{kampezidou2016distribution}.


\begin{assumption}\label{ass:pros_sched_demand}
Each prosumer $i {\in} \{1,2,\cdots,N\}$ holds a day-ahead demand schedule $\mathbf{x}_i^{\o{}} {\in} \mathbb{R}^T$, with each hourly demand $\infty {>} x_{it}^{\o{}} {\geq} 0$, $\forall t {\in} \mathbb{T}=\{1,\cdots,T\}$. Furthermore, $\exists$ at least one $t$ for which $x_{it}^{\o{}} {>}0$ so that the problem is meaningful.\frqed
\end{assumption}

\begin{assumption}\label{ass:pros_dependence}
The prosumers' actions depend on the DR-aggregator's action $\mathbf{x}_i(\mathbf{p}) = [x_{i1}(\mathbf{p}),\cdots, x_{iT}(\mathbf{p})]^\mathsf{T}$ in Stackelberg games. A similar assumption was made in \cite{Yang2019AdaptiveLI}.\frqed

\end{assumption}
Consider heterogeneous prosumers that differ by an inelasticity parameter $0{<}u_i{<}\infty$ that represents their sensitivity in deviations from schedule. Moreover, each prosumer $i$ can choose her new total daily demand $W_i {\in}  [0, a \sum_{t=1}^T x_{it}^{\o{}}]$, which may be different that the initially scheduled $\sum_{t=1}^T x_{it}^{\o{}}$. Note that $a {\in} [1, + \infty)$. An increase or decrease in demand $x_{it}$, compared to scheduled demand $x_{it}^{\o{}}$, is interpreted as if prosumer $i$ purchasing or selling energy at hour $t$. The total daily demand that $i$ trades is $Q_i {=} \sum_{t=1}^{T} x_{it}^{\o{}} - W_i$, where $Q_i {\in} [(1{-}a)\sum_{t=1}^T x_{it}^{\o{}},\sum_{t=1}^T x_{it}^{\o{}}]$. The prosumer's problem is,
\begin{equation}
     \max_{\mathbf{x}_i} g_{p_{i}} = {\mathbf{p}}^\mathsf{T} (\mathbf{x}_i^{\o{}}-\mathbf{x}_i(\mathbf{p})) - \mathds{1}^\mathsf{T} \mathbf{V}_i(\mathbf{x}_i^{\o{}},\mathbf{x}_i(\mathbf{p})),
    \label{pros_obj}
\end{equation}
such that,
\begin{equation}
    \mathds{1}^\mathsf{T} \mathbf{x_i}(\mathbf{p}) = W_{i} ,
    \label{pros_eq_con}
\end{equation}
\begin{equation}
    0 \leq x_{it}(\mathbf{p}) \leq a x_{it}^{\o{}}, \forall t \in \{1,2,\cdots,T\},
    \label{pros_in_con}
\end{equation}
\noindent
where $\mathbf{V}_i \in \mathbb{R}^T$ is the vector of utility functions $\mathbf{V}_i = [u_i (x_{i1}^{\o{}} - x_{i1}(\mathbf{p}))^2, ..., u_i (x_{iT}^{\o{}} - x_{iT}(\mathbf{p}))^2]^\mathsf{T}$. Note that, when $x_{it}(\mathbf{p}){>}x_{it}^{\o{}}$, the first term for that hour $t$ in (\ref{pros_obj}) becomes negative because the prosumer buys extra energy and the inconvenience cost is negative too. However, this framework offers the flexibility to increase the total daily payoff if the prosumer can sell this purchased energy in a higher $\mathbf{p}$ hour (energy arbitrage). When $x_{it}{<}x_{it}^{\o{}}$, the first term in (\ref{pros_obj}) is positive and the second term is negative but there is a feasible space region (\ref{pros_eq_con}), (\ref{pros_in_con}) where the difference of these two is positive for that hour $t$. Hours with negative payoffs are balanced by hours of positive payoffs since (\ref{pros_obj}) is maximized.

\begin{remark}\label{rem:pros_conv_opt}
The problem (\ref{pros_obj}), (\ref{pros_eq_con}), (\ref{pros_in_con}) has a strictly concave objective ($u_i {>} 0$) w.r.t. $\mathbf{x}_i$ for given $\mathbf{p}$ and a convex constraint set ((\ref{pros_eq_con}) is affine, (\ref{pros_in_con}) is closed, bounded and convex). Hence, the KKT conditions are necessary and sufficient \cite{nemir_convex_opt}.
\frqed
\end{remark}

\subsection{The aggregator's problem}
The DR-aggregator's action $\mathbf{p} {\in} \mathbb{R}^T$ consists of bounded prices $0 {\leq} p_t {\leq} p^\mathrm{max} {<} \infty, \forall t {\in} \{1,2,\cdots,T\}$. Without loss of generality, choose $p^\mathrm{max}{=}\min \{\lambda_1,\cdots,\lambda_T\}$. Then, DR-aggregator's objective is to maximize her daily payoff $\sum_{t=1}^{T} (\lambda_t \sum_{i=1}^{N} (x_{it}^{\o{}} {-} x_{it}(\mathbf{p})) {-} p_t \sum_{i=1}^N (x_{it}^{\o{}} {-} x_{it}(\mathbf{p})))$, i.e.,
\begin{equation}
    \max_{\mathbf{p}} g_a = (\bm{\lambda} - \mathbf{p})^\mathsf{T} \mathbf{\Delta d}(\mathbf{p}),
    \label{agg_obj}
\end{equation}
such that,
\begin{equation}
    0 \leq p_t \leq p^\mathrm{max}, \forall t \in \{1,2,\cdots,T\} ,
    \label{agg_in_con}
\end{equation}
\noindent
for some vectors $\mathbf{\Delta d} \in \mathbb{R}^T$, $\mathbf{d}^{\o{}} \in \mathbb{R}^T$ and $\mathbf{d} \in \mathbb{R}^T$. Note that $\mathbf{\Delta d}(\mathbf{p}) = \mathbf{d}^{\o{}}-\mathbf{d}(\mathbf{p})$ and $d_{t}^{\o{}} = \sum_{i=1}^N x_{it}^{\o{}}$, $d_{t}(\mathbf{p}) = \sum_{i=1}^N x_{it}(\mathbf{p})$.
According to (\ref{agg_obj}), the DR-aggregator purchases $\Delta d_t >0$ total hourly energy from all her prosumers at price $p_t$ to sell it to the day-ahead energy market at a price $\lambda_t$. If $\Delta d_t <0$, then the demand for this hour is increased and the DR-aggregator will purchase energy from the market to sell to her prosumers. Because assumption $p^\mathrm{max} = \min\{\lambda_1,\cdots,\lambda_T\}$ holds, it is true that $\lambda_t \geq p_t$, $\forall t \in \{1,\cdots,T\}$ and hence the DR-aggregator's payoff would be negative when $\Delta d_t < 0$, but only for that particular hour $t$. However, the DR-aggregator maximizes her total daily payoff which results in minimizing the negative payoff hours, outbalancing them with positive ones so that the total daily payoff is the maximum possible.

\begin{remark}\label{rem:KKT_no_solve}
The game (\ref{pros_obj}), (\ref{pros_eq_con}), (\ref{pros_in_con}), (\ref{agg_obj}), (\ref{agg_in_con}), has been proven in Theorem 2 of \cite{kampezidou2021online} to have at least one equilibrium not recoverable in closed-form due KKT coupling. Coupled KKT equations have also been reported in differential games \cite{mylvaganam2014approximate}.
\end{remark}

Following Remark \ref{rem:KKT_no_solve}, an online decentralized, privacy-preserving solution is proposed for the recovery of approximate Stackelberg equilibria. 

\section{Scalable Decentralized Privacy-Preserving Equilibrium Solution Learning}\label{sec:onl_learn}
In this section, the game (\ref{pros_obj}), (\ref{pros_eq_con}), (\ref{pros_in_con}), (\ref{agg_obj}), (\ref{agg_in_con}) is decoupled by recovering the mapping $\mathbf{d}^{\star}(\mathbf{p}): \mathbb{P} \rightarrow \mathbb{Y}$, where $\mathbf{d}^{\star}(\mathbf{p}) = \sum_{i=1}^N x_{it}^{\star}(\mathbf{p}) \in \mathbb{R}^T$. This mapping can be found by solving problem (\ref{LS_Problem}), with the method proposed in this section. This method requires basis assumptions, summarized in Assumption \ref{ass:PWL_basis}. Similar mappings have been used before in bilevel programming \cite{sinha2017evolutionary} with a genetic algorithm, but without solution quality bounds, as those in section \ref{sec:onl_learn} next.

\begin{assumption}\label{ass:PWL_basis}
Assume a basis $\bm{\phi}(\mathbf{p}) \in \mathbb{R}^K$, with $K < \infty$, which is Lipschitz continuous on $\mathbb{P}$ with constant $L_{\phi} \geq 0$ and G\^{a}teaux differentiable on $\mathbb{P}$, with a Lipschitz continuous G\^{a}teaux derivative on $\mathbb{P}$ modulus $L_{d \phi} \geq 0$, where $\mathbb{P}$ is the compact feasible set (\ref{agg_in_con}) of the DR-aggregator's optimization problem. Moreover, assume that $\|\bm{\phi}(\mathbf{p})\|_2 \leq \phi^\mathrm{max} < \infty$, $\forall \mathbf{p} \in \mathbb{P}$. Since $\bm{\phi}(\mathbf{p})$ is Lipschitz continuous and G\^{a}teaux  differentiable on $\mathbb{P}$, it is also Fr\'{e}chet differentiable on $\mathbb{P}$ and it has continuous partial derivatives on $\mathbb{P}$, which are also Lipschitz continuous on $\mathbb{P}$ modulus $L_{\partial \phi} \geq 0$ \cite{lau1978differentiability}. \frqed
\end{assumption}

Given the basis $\bm{\phi}(\mathbf{p})$ in Assumption \ref{ass:PWL_basis} and a solution mapping $\mathbf{d}^{\star}(\mathbf{p}) \in \textrm{span}\{\bm{\phi}(\mathbf{p})\}$, the game (\ref{pros_obj}), (\ref{pros_eq_con}), (\ref{pros_in_con}), (\ref{agg_obj}), (\ref{agg_in_con}) is equivalent to the following decoupled problems:
\begin{equation}
\begin{split}
     \bm{\Theta_c}^{\star} {=} \argmin_{\bm{\Theta_c}}
     \int_{\mathbf{p} \in \mathbb{P}} \|\mathbf{d}^{\star}(\mathbf{p}) - \bm{\Theta_c}^{\mathsf{T}} \bm{\phi}(\mathbf{p})\|^2_2 \textrm{d}\mathbf{p},
    \label{LS_Problem}
\end{split}
\end{equation}
\begin{equation}
\begin{split}
    \mathbf{p}^s \in \argmax_{\mathbf{p}} 
    \big\{& g_a = ( \bm{\lambda} - \mathbf{p})^\mathsf{T} (\mathbf{d}^{\o{}}-\bm{\Theta_c}^{{\star}^{\mathsf{T}}} \bm{\phi}(\mathbf{p})):\\
    & 0 \leq p_t \leq p^\mathrm{max}, \forall t \in \{1,2,\cdots,T\} \big\},\\
    \label{agg_learning_inf}
\end{split}
\end{equation}
\begin{equation}
\begin{split}
    \mathbf{x}_i^s = \argmax_{\mathbf{x}_i} 
    \big\{ & g_{p_i} = {\mathbf{p}^s}^\mathsf{T} (\mathbf{x}_i^{\o{}}-\mathbf{x}_i) - \mathds{1}^\mathsf{T} \mathbf{V}_i(\mathbf{x}_i^{\o{}},\mathbf{x}_i):\\
    & \mathds{1}^\mathsf{T} \mathbf{x_i} = W_{i},\\
    & 0 \leq x_{it} \leq a x_{it}^{\o{}}, \forall t \in \{1,2,\cdots,T\} \big\},
    \label{pros_learning_inf}
\end{split}
\end{equation}
\noindent
where $\bm{\Theta_c}^{\star} \in \mathbb{R}^{K \times T}$. A unique minmizer $\bm{\Theta_c}^{\star}$ to (\ref{LS_Problem}) exists because of convexity. Therefore, if (\ref{LS_Problem}) can be learned with an online method, then problems (\ref{agg_learning_inf}) and (\ref{pros_learning_inf}), can be solved to equilibrium. In particular, problem (\ref{agg_learning_inf}) would admit at least one solution, although multiple can occur depending on the choice of $\bm{\phi}(\mathbf{p})$. On the other hand, problem (\ref{pros_learning_inf}) would admit a unique solution for each $\mathbf{p}^s$ (proved in Theorem 1 of \cite{kampezidou2021online} via strict convexity). Therefore, multiple Stackelberg equilibria $(\mathbf{x}_i^s,\mathbf{p}^s)$ can occur with same value $g_a(\mathbf{p}^s;\bm{\Theta_c}^{\star})$.

The section aims to establish a method that identifies the impact of learning $\mathbf{d}^{\star}(\mathbf{p})$ when $\mathbf{d}^{\star}(\mathbf{p}) \notin \textrm{span}\{\bm{\phi}(\mathbf{p})\}$, which is the most general case when the equilibrium strategy forms are unknown (Remark \ref{rem:KKT_no_solve}). When $\mathbf{d}^{\star}(\mathbf{p}) \notin \textrm{span}\{\bm{\phi}(\mathbf{p})\}$, an approximation error occurs (Lemma \ref{lemma:approx_error}) which leads to $\epsilon$-Stackelberg solutions of the game (\ref{pros_obj}), (\ref{pros_eq_con}), (\ref{pros_in_con}), (\ref{agg_obj}), (\ref{agg_in_con}) with bounded players' utility losses, as proved in Theorem \ref{thm:model_mismatch} later.

\begin{lemma}\label{lemma:approx_error}
Consider the basis $\bm{\phi}(\mathbf{p}) \in \mathbb{R}^K$, $K {<} \infty$ and assume that $\mathbf{d}^{\star}(\mathbf{p}) {\notin} \textrm{span}\{\bm{\phi}(\mathbf{p})\}$. Then, the estimation problem,
\begin{equation}
\begin{split}
     \bm{\Theta}^{\star} {=} \argmin_{\bm{\Theta}}
     \int_{\mathbf{p} \in \mathbb{P}} \|\mathbf{d}^{\star}(\mathbf{p}) - \bm{\Theta}^{\mathsf{T}} \bm{\phi}(\mathbf{p})\|^2_2 \textrm{d}\mathbf{p},
    \label{LS_Problem_not_span}
\end{split}
\end{equation}
has a bounded approximation error $\bm{\epsilon}_m(\mathbf{p}) {\in} \mathbb{R}^T$, i.e., $\|\bm{\epsilon}_m(\mathbf{p})\|_2 {\leq} \epsilon_m^\mathrm{max} {<} \infty$, $\forall \mathbf{p} \in [0, p^\mathrm{max}]$, resulting in a bounded estimation of $\mathbf{d}^{\star}(\mathbf{p}) {=} \bm{\Theta}^{{\star}\mathsf{T}} \bm{\phi}(\mathbf{p}) + \bm{\epsilon}_m(\mathbf{p})$.
\end{lemma}
\begin{proof}
Given that all players' actions are bounded in compact sets (\ref{pros_eq_con}), (\ref{pros_in_con}) and (\ref{agg_in_con}), Weierstrass approximation theorem \cite{achieser2013theory} holds and a bounded approximation error occurs, i.e., $\|\bm{\epsilon}_m(\mathbf{p})\| {\leq} \epsilon_m^\mathrm{max} {<} \infty$, $\forall  \mathbf{p} \in [0, p^\mathrm{max}]$. Therefore, the distance between the estimation of the sum of the prosumers' best responses and the true best response function is bounded by $\|\mathbf{d}^{\star}(\mathbf{p}) {-} \bm{\Theta}^{{\star}\mathsf{T}} \bm{\phi}(\mathbf{p})\|_2= \|\bm{\epsilon}_m(\mathbf{p})\|_2 \leq \epsilon_m^\mathrm{max}$. \frQED
\end{proof}
\begin{assumption}\label{ass:bounded_theta_r}
For the bounded approximation error $\bm{\epsilon}_m(\mathbf{p}) \in \mathbb{R}^T$, i.e., $\|\bm{\epsilon}_m(\mathbf{p})\|_2 = \|\mathbf{d}^{\star}(\mathbf{p}) - \bm{\Theta}^{{\star}\mathsf{T}} \bm{\phi}(\mathbf{p})\|_2 = \|\bm{\Theta_c} ^{{\star}\mathsf{T}} \bm{\phi}(\mathbf{p}) - \bm{\Theta}^{{\star}\mathsf{T}} \bm{\phi}(\mathbf{p})\|_2 = \|(\bm{\Theta_c} ^{{\star}} - \bm{\Theta}^{\star})^{\mathsf{T}} \bm{\phi}(\mathbf{p})\|_2 \leq \epsilon_m^\mathrm{max}$ of Lemma \ref{lemma:approx_error}, and under Assumption \ref{ass:PWL_basis} on the boundedness of the basis $\bm{\phi}(\mathbf{p}) {\in} \mathbb{R}^K$, $K {<} \infty$, i.e., $\|\bm{\phi}(\mathbf{p})\|_2 {\leq} \phi^\mathrm{max} {<} \infty$, $\forall \mathbf{p} {\in} \mathbb{P}$, assume $\|\bm{\Theta_c} ^{{\star}} {-} \bm{\Theta}^{\star}\|_F {=} \|\bm{\Theta_r}\|_F {\leq} \theta_r^\mathrm{max} {<} \infty$ is also bounded. \frqed
\end{assumption}

An online method is presented next to solve (\ref{LS_Problem_not_span}) by sampling feasible DR-aggregator actions $\mathbf{p}_r$, communicating them to the prosumers via a noiseless channel and collecting only their best responses $\mathbf{x}_i^{\star}(\mathbf{p}_r)$ (privacy-preserving of their problem details), after they individually solve (\ref{pros_obj}), (\ref{pros_eq_con}), (\ref{pros_in_con}) in a decentralized, parallelized sense. Consider the problem of estimating the unknown function $\mathbf{d}^{\star}(\mathbf{p}): \mathbb{P} \rightarrow \mathbb{Y}$ from iid samples $\mathbf{p}^1_r,\mathbf{p}^2_r,...,\mathbf{p}^J_r$. Without any prior information on $\mathbf{d}^{\star}(\mathbf{p})$, uniform sampling can be assumed, i.e., $p_{r_t}^j$\url{~}Uni$[0,p^\mathrm{max}]$, $\forall t \in \{1,...,T\}$. The DR-aggregator sends a feasible action $p_{r_t}^j$\url{~}Uni$[0,p^\mathrm{max}]$, $\forall t \in \{1,...,T\}$ to the prosumers and collects their responses $\mathbf{x}_i^{j{\star}}(\mathbf{p}_r^j)$. Hence, the learning problem (\ref{LS_Problem_not_span}) can be approximated by the Least Squares (LS) problem that utilizes the online collected samples $(\mathbf{x}_i^{j{\star}}(\mathbf{p}^j_r), \mathbf{p}^j_r)$. 


Now, recall that when $\mathbf{d}^{\star}(\mathbf{p}) \notin \textrm{span}\{\bm{\phi}(\mathbf{p})\}$, it is true that $\mathbf{d}^{\star}(\mathbf{p}) = \bm{\Theta}^{{\star}\mathsf{T}} \bm{\phi}(\mathbf{p}) + \bm{\epsilon}_m(\mathbf{p})$, where the approximation error $\bm{\epsilon}_m(\mathbf{p})$ in Lemma \ref{lemma:approx_error} is unknown and therefore $\bm{\Theta}^{\star}$ cannot be recovered. An estimator $\bm{\hat{\Theta}} \in \mathbb{R}^{K \times T}$ is proposed for $\bm{\Theta}^{\star}$ along with the LS problem (\ref{RLS_Problem_appr}) to recover $\bm{\hat{\Theta}}^J$, the optimal value of $\bm{\hat{\Theta}}$, from $J$ online collected tuples $(\mathbf{x}_i^{j{\star}}(\mathbf{p}^j_r), \mathbf{p}^j_r)$. Note that (\ref{RLS_Problem_appr}) is not the only way to estimate $\bm{\Theta}^{\star}$ and variations in formulation (\ref{RLS_Problem_appr}) can potentially offer different upper bounds to $\|\bm{\Theta}^{\star} {-} \bm{\hat{\Theta}}^J\|_F$, than the one derived in Theorem \ref{thm:UUB_PE_with_wo_appr_error} next. This upper bound is necessary for Theorem \ref{thm:model_mismatch} later. Hence consider,
\begin{equation}
\begin{split}
    \bm{\hat{\Theta}}^J = \argmin_{\bm{\hat{\Theta}}} \big\{ & \frac{1}{2} \sum_{j=1}^J m^{J-j} \|\mathbf{d}^{j{\star}}(\mathbf{p}_r^j) {-} \bm{\hat{\Theta}}^{\mathsf{T}} \bm{\phi}(\mathbf{p}_r^j)\|_2^2 \\
    + & \frac{1}{2} m^J \tr \{  (\bm{\hat{\Theta}}-\bm{\hat{\Theta}}_0)^\mathsf{T} \bm{\Pi}_0^{-1}(\bm{\hat{\Theta}}-\bm{\hat{\Theta}}_0) \} \big\},
    \label{RLS_Problem_appr}
\end{split}
\end{equation}
\noindent
where $\bm{\phi}^j(\mathbf{p}_r^j) \in \mathbb{R}^K$, $\bm{\hat{\Theta}} \in \mathbb{R}^{K \times T}$, $\bm{\hat{\Theta}}_0 \in \mathbb{R}^{K \times T}$, $\bm{\Pi}_0^{-1} \in \mathbb{R}^{K \times K}$, $0< m < 1$ are forgetting factors to some power $J-j$ \cite{bruggemann2021exponential} and $\frac{1}{2} m^J \tr \{(\bm{\hat{\Theta}}-\bm{\hat{\Theta}}_0)^\mathsf{T} \bm{\Pi}_0^{-1}(\bm{\hat{\Theta}}-\bm{\hat{\Theta}}_0) \}$ is a regularization term proposed in \cite{ioannou2006adaptive}, that ensures invertibility of the matrix $\bm{\Phi} \mathbf{M} \bm{\Phi}^{\mathsf{T}}$ in the closed form (\ref{LS_law_theta}). It can proved then (proof similar to \cite{bruggemann2021exponential} omitted for brevity), that (\ref{RLS_Problem_appr}) has a unique closed form solution for $K \leq J$,
\begin{equation}
    \bm{\hat{\Theta}}^J {=} (\bm{\Phi} \mathbf{M} \bm{\Phi}^{\mathsf{T}} {+} m^J \bm{\Pi}_0^{-1})^{-1} (\bm{\Phi} \mathbf{M} \mathbf{D}^{{\star}} {+} m^J \bm{\Pi}_0^{-1} \bm{\hat{\Theta}}_0),
    \label{LS_law_theta}
\end{equation}
where $\bm{\Phi} \in \mathbb{R}^{K \times J}$ contains vectors $\bm{\phi}^j(\mathbf{p}_r^j)$ and $\mathbf{D}^{\star} \in \mathbb{R}^{J \times T}$ contains vectors $\mathbf{d}^{j{\star}}(\mathbf{p}_r^j)$. The diagonal matrix $\mathbf{M} \in \mathbb{R}^{J \times J}$, contains the values $\sqrt{m^{J-j}}$, for each sample $j$ collected.

The DR-aggregator has 24 hours to solve the game as accurately as possible. The more samples $J$ are collected, the more accurate the estimation will be without any additional information on $\mathbf{d}^{\star}(\mathbf{p})$. As $J {\rightarrow} \infty$, solving (\ref{LS_law_theta}) becomes hard because a $(K {\times} J){\times}(J {\times} K)$ multiplication is necessary. Such multiplication and the inversion of a large matrix $\bm{\Phi} \mathbf{M} \bm{\Phi}^{\mathsf{T}}$ can be avoided with Recursive Least Squares (RLS), which are known to learn deterministic signals fast \cite{sayed2003fundamentals}. As $J {\rightarrow} \infty$, the learning error is minimized (part 2 of Theorem \ref{thm:UUB_PE_with_wo_appr_error}). 


It is well-known (chapter 4.6.1 in \cite{ioannou2006adaptive}) that (\ref{LS_law_theta}) admits an RLS solution, in the absence of approximation errors, i.e., $\bm{\epsilon}_m(\mathbf{p}){=}0$. It was recently shown \cite{bruggemann2021exponential}, that an RLS form of (\ref{LS_law_theta}) converges to a ball around $\bm{\Theta}^{\star}$ in the presence of approximation errors, i.e., $\bm{\epsilon}_m(\mathbf{p}) {\neq} 0$. Since in our application, the system's $\mathbf{d}^{\star}(\mathbf{p}) {=} \bm{\Theta}^{{\star}\mathsf{T}} \bm{\phi}(\mathbf{p}) + \bm{\epsilon}_m(\mathbf{p})$ dimensions are different than \cite{bruggemann2021exponential}, the resulting RLS laws (\ref{RLS_theta}), (\ref{RLS_epsilon}), (\ref{RLS_pi}) are slightly different than \cite{bruggemann2021exponential} (proof omitted for brevity). Our RLS laws do not require a $(T \times T)$ matrix inversion, as opposed to \cite{bruggemann2021exponential}, since $m {+} \bm{\phi}^{j^\mathsf{T}} \bm{\Pi}^{j-1} \bm{\phi}^j$ in (\ref{RLS_theta}) is scalar,
\begin{equation}
    \bm{\hat{\Theta}}^j = \bm{\hat{\Theta}}^{j-1} +  \bm{\Pi}^{j-1} \bm{\phi}^j (m+\bm{\phi}^{j^\mathsf{T}} \bm{\Pi}^{j-1} \bm{\phi}^j)^{-1} \bm{\epsilon}^j,
    \label{RLS_theta}
\end{equation}
\begin{equation}
    \bm{\epsilon}^j = \mathbf{d}^{{j{\star}}^\mathsf{T}} - \bm{\phi}^{j^\mathsf{T}} \bm{\hat{\Theta}}^{j-1},
    \label{RLS_epsilon}
\end{equation}
\begin{equation}
    \bm{\Pi}^j = \frac{1}{m} (\mathbb{I} - \bm{\Pi}^{j-1} \bm{\phi}^j (m+\bm{\phi}^{j^\mathsf{T}} \bm{\Pi}^{j-1} \bm{\phi}^j)^{-1} \bm{\phi}^{j^\mathsf{T}}) \bm{\Pi}^{j-1},
    \label{RLS_pi}
\end{equation}
\begin{equation}
    \bm{\Pi}^0 = \bm{\Pi}_0 = \bm{\Pi}_0^\mathsf{T} \succ 0, \bm{\hat{\Theta}}^0 = \bm{\Theta}_0,
    \label{RLS_initial_cond}
\end{equation}
\noindent
where $\bm{\phi}^j(\mathbf{p}^j) {\in} \mathbb{R}^K$, $\bm{\hat{\Theta}}^j {\in} \mathbb{R}^{K \times T}$, $\bm{\Pi}^j {\in} \mathbb{R}^{K \times K}$, $m^{J-j} {\in} \mathbb{R}$ and $\bm{\Theta}_0 {\in} \mathbb{R}^{K \times T}$, $\bm{\Pi}_0^{-1} {\in} \mathbb{R}^{K \times K}$ are constant matrices in the bias term. The effect of the regularization term diminishes over $j$, due to the exponent $J{-}j$ of the weight $0 {<} m {<} 1$, in (\ref{RLS_Problem_appr}), which is smaller for newer samples, resulting in higher $m^{J-j}$ values for the more recent samples. Regularization terms are necessary to extract and initialize RLS laws. Note that the unique solutions $\bm{\hat{\Theta}}^j$ from (\ref{LS_law_theta}) and (\ref{RLS_theta}) are the same, when $j{=}J$. Theorem \ref{thm:UUB_PE_with_wo_appr_error} provides an upper bound to the learning error $\mathbf{\tilde \Theta}^j {=} \mathbf{\Theta^{\star}} {-} \mathbf{\hat{\Theta}}^j$, which will be used by Theorem \ref{thm:model_mismatch}, to bound the unknown $g_a(\mathbf{d}^{\star}(\mathbf{p}^s){,}\mathbf{p}^s)$ and $g_{p_i}(\mathbf{x}_i^{\star}(\mathbf{p}^{s}){,}\mathbf{p}^{s})$.

\begin{theorem}\label{thm:UUB_PE_with_wo_appr_error}
Given the update laws (\ref{RLS_theta}), (\ref{RLS_epsilon}), (\ref{RLS_pi}), (\ref{RLS_initial_cond}) and a persistently excited (PE) signal $\bm{\phi}^j(\mathbf{p}^j_r)$, i.e., $\infty > \beta_1 \mathbf{I} \geq \sum_{j=k}^{k+M} \bm{\phi}^j(\mathbf{p}^j_r) \bm{\phi}^{j^\mathsf{T}}(\mathbf{p}^j_r) \geq \beta_0 \mathbf{I} > 0$, $\beta_0 >0$, $\beta_1 > 0$, $M>0$, the following hold $\forall j \geq M$,
\begin{enumerate}
    \item Without approximation errors, i.e., $\bm{\epsilon}_m(\mathbf{p})=0$ and $\bm{\Theta_c}^{\star} = \mathbf{\Theta ^\star}$, it holds that $\mathbf{\hat{\Theta}}^j \rightarrow \mathbf{\Theta^{\star}}$ exponentially and therefore the sum of best responses $\mathbf{d}^{\star}(\mathbf{p}) = \bm{\Theta_c}^{{\star}^{\mathsf{T}}} \bm{\phi}(\mathbf{p})$ is recovered perfectly, as $j \rightarrow \infty$.
    \item With approximation errors, i.e., $\bm{\epsilon}_m(\mathbf{p}) \neq \mathbf{0}$ and $\bm{\Theta_c}^{\star} \neq \mathbf{\Theta ^\star}$, it holds that $\mathbf{\tilde \Theta}^j = \mathbf{\Theta^{\star}} - \mathbf{\hat{\Theta}}^j$ converges exponentially to a ball of $\mathbf{\hat{\Theta}}^j$ which contains $\mathbf{\Theta}^{\star}$ and therefore, $\bm{\hat{\Theta}}^{j^{\mathsf{T}}} \bm{\phi}(\mathbf{p})$ only approximates the sum of best responses $\mathbf{d}^{\star}(\mathbf{p})$, as $j \rightarrow \infty$. It is also true that $\lim_{j \rightarrow \infty} \|\mathbf{\tilde \Theta}^{j}\|_F \leq \eta \epsilon_m^\mathrm{max}$, where $\eta$ is a positive constant defined as,
    \begin{equation*}
        \eta = \left(\frac{m^{-(M+1)}{-}1}{\beta_0 (m^{-1}{-}1)}\right)^{3/2} \frac{\sqrt{\beta_1 \lambda_\textrm{max}(\bm{\Pi}^{0^{-1}})}}{1{-}\sqrt{m}},
    \end{equation*}
    \noindent
    and $\lambda_\textrm{max}(\bm{\Pi}^{0^{-1}}) > 0$ is the maximum eigenvalue.
\end{enumerate}
\end{theorem}
\begin{proof}
See Appendix \ref{sec:proof_thm_UUB_PE_with_wo_appr_error}.
\frQED
\end{proof}

According to Theorem \ref{thm:UUB_PE_with_wo_appr_error}, a tighter PE bound (large $\beta_0$ and small $\beta_1$), a small PE window $M$, and a small $\lambda_\textrm{max}(\bm{\Pi}^{0^{-1}})$ can reduce the value of $\eta$ and contribute to a tighter $\|\mathbf{\tilde \Theta}^{j}\|_F$ bound. The sensitivity of the bound to $m$ is described by the complicated relationship $\eta(m)$ which if plotted, will be found convex with a minimum in the $0 {<} m {<} 1$ range.


Theorem \ref{thm:UUB_PE_with_wo_appr_error} proved that without approximation errors, exact Stackelberg solutions $(\mathbf{x}_i^{s}, \mathbf{p}^{s})$ can be recovered after solving (\ref{agg_learning_inf}) and (\ref{pros_learning_inf}) with $\mathbf{\hat{\Theta}}^j {\rightarrow} \bm{\Theta_c}^{\star}$ and that with approximation errors, $\epsilon$-Stackelberg solutions $(\mathbf{x}_i^{\epsilon s}, \mathbf{p}^{\epsilon s})$ occur by sequentially solving the perturbed problems (\ref{agg_learning_appr}), (\ref{pros_learning_appr}), with $\bm{\hat{\Theta}}^J {\neq} \bm{\Theta_c}^{\star}$,
\begin{equation}
\begin{split}
    \mathbf{p}^{\epsilon s} \in \argmax_{\mathbf{p}}
    \big\{ & \tilde{g}_a = (\bm{\lambda} - \mathbf{p})^\mathsf{T} (\mathbf{d}^{\o{}}-\bm{\hat{\Theta}}^{J^{\mathsf{T}}} \bm{\phi}(\mathbf{p})):\\
    & 0 \leq p_t \leq p^\mathrm{max}, \forall t \in \{1,2,\cdots,T\} \big\},\\
    \label{agg_learning_appr}
\end{split}
\end{equation}
\begin{equation}
\begin{split}
    \mathbf{x}_i^{\epsilon s} = \argmax_{\mathbf{x}_i}
    \big\{ & g_{p_i} = {\mathbf{p}^{\epsilon s}}^\mathsf{T} (\mathbf{x}_i^{\o{}}-\mathbf{x}_i) - \mathds{1}^\mathsf{T} \mathbf{V}_i(\mathbf{x}_i^{\o{}},\mathbf{x}_i):\\
    & \mathds{1}^\mathsf{T} \mathbf{x_i} = W_{i},\\
    & 0 \leq x_{it} \leq a x_{it}^{\o{}}, \forall t \in \{1,2,\cdots,T\} \big\}.
    \label{pros_learning_appr}
\end{split}
\end{equation}

Existence and non-uniqueness of $\mathbf{p}^{\epsilon s}$ in (\ref{agg_learning_appr}), uniqueness of $\mathbf{x}_i^{\epsilon s} {=} \mathbf{x}_i^{\star}(\mathbf{p}^{\epsilon s})$ in (\ref{pros_learning_appr}), for every $\mathbf{p}^{\epsilon s}$ of (\ref{agg_learning_appr}) and continuity of $\tilde{g}_a(\mathbf{p}{;}\bm{\hat{\Theta}}^J)$ on the compact set $\mathbb{P}$ under Assumption \ref{ass:PWL_basis}, can be proved similarly to Theorem 1 of \cite{kampezidou2021online}. Non-uniqueness of $\mathbf{p}^{\epsilon s}$ occurs from the non-linearity of $\tilde{g}_a(\mathbf{p};\bm{\hat{\Theta}}^{J})$, since the definiteness of $\frac{\partial^2 \tilde g_{a}(\mathbf{p}{;}\bm{\hat{\Theta}}^J)}{\partial \mathbf{p}^2} {=} -(\bm{\lambda} {-} \mathbf{p})^\mathsf{T}\bm{\hat{\Theta}}^{J^{\mathsf{T}}} \frac{\partial^2 \bm{\phi}(\mathbf{p})}{\partial \mathbf{p}^2} {+} (\frac{\partial \bm{\phi}(\mathbf{p})}{\partial \mathbf{p}})^{\mathsf{T}} \bm{\hat{\Theta}}^{J} {+} \bm{\hat{\Theta}}^{J^{\mathsf{T}}} \frac{\partial \bm{\phi}(\mathbf{p})}{\partial \mathbf{p}}$ can vary with the $\bm{\phi}(\mathbf{p})$ and the definiteness of $\bm{\hat{\Theta}}^{J}$. Although it is possible to ensure definiteness of $\frac{\partial^2 \tilde g_{a}(\mathbf{p}{;}\bm{\hat{\Theta}}^J)}{\partial \mathbf{p}^2}$, i.e., concavity of $\tilde g_{a}(\mathbf{p}{;}\bm{\hat{\Theta}}^J)$, through the proper choice of $\bm{\phi}(\mathbf{p})$ and a projection-based algorithm for $\bm{\hat{\Theta}}^{J}$, the restriction of $\bm{\hat{\Theta}}^{J}$ in a specific subset would not necessarily ensure the most accurate learning of $\mathbf{d}^{\star}(\mathbf{p})$, which is the scope of this work.

\SetAlFnt{\small\sf}
\SetAlCapFnt{\small}
\SetAlTitleFnt{\small}
\SetCommentSty{\small}
\SetKwInput{KwInit}{Initialize}
\RestyleAlgo{ruled}
\begin{algorithm}[b!]
\algsetup{linenosize=\tiny}
\SetAlgoLined
\KwInit{$\mathbf{x}_i^{\o{}}$, $W_i$, $\bm{\Pi}^0 = \bm{\Pi}_0 \succ 0$, $\bm{\hat{\Theta}}^0 = \bm{\Theta}_0$}
\For{$j = 1$ \KwTo $J$}{
    $\mathbf{p}_{r_t}^j $ \textrm{\url{~}Uni}$[0,p^\mathrm{max}]$, $\forall t \in \{1,...,T\}$\\
    $\mathbf{x}_i^{j\star} = \argmax_{\mathbf{x}_i^j \in \mathbb{X}_i} g_{pi} (\mathbf{x}_i^j; \mathbf{p}^j_r, \mathbf{x}_i^{\o{}},W_i)$ \\
    $\mathbf{d}^{j{\star}} = \sum_{i=1}^N \mathbf{x}_i^{j\star}$\\
    $\bm{\epsilon}^j = \mathbf{d}^{{j{\star}}^\mathsf{T}} - \bm{\phi}^j(\mathbf{p}^j_r)^\mathsf{T} \bm{\hat{\Theta}}^{j-1}$ \\
    $\bm{\hat{\Theta}}^j = \bm{\hat{\Theta}}^{j-1} +  \bm{\Pi}^{j-1} \bm{\phi}^j (m+\bm{\phi}^{j^\mathsf{T}} \bm{\Pi}^{j-1} \bm{\phi}^j)^{-1} \bm{\epsilon}^j$\\
    $\bm{\Pi}^j = \frac{1}{m} (\mathbb{I} - \bm{\Pi}^{j-1} \bm{\phi}^j (m+\bm{\phi}^{j^\mathsf{T}} \bm{\Pi}^{j-1} \bm{\phi}^j)^{-1} \bm{\phi}^{j^\mathsf{T}}) \bm{\Pi}^{j-1}$\\
    }
    
$\mathbf{p}^{\epsilon s} = \argmax_{\mathbf{p} \in \mathbb{P}} \tilde g_a(\mathbf{p};\mathbf{d}^{\o{}}, \bm{\lambda}, \bm{\hat{\Theta}}^{J})$ \\
$\mathbf{x}_i^{\epsilon s} = \argmax_{\mathbf{x}_i \in \mathbb{X}_i} g_{pi} (\mathbf{x}_i; \mathbf{p}^{\epsilon s}, \mathbf{x}_i^{\o{}},W_i)$ \\
\textrm{Submit bid} $(\bm{\lambda}, \mathbf{d}^{\o{}}{-}\bm{\hat{\Theta}}^{J^\mathsf{T}} \bm{\phi}(\mathbf{p}^{\epsilon s}))$ \textrm{to market}
 \caption{Decentralized Privacy-Preserv. Learning \\ of $\epsilon$-Stackelberg Equilibria in Energy Trading Games}
 \label{algo:stackelberg_algo_online}
\end{algorithm}

Since $g_a(\mathbf{d}^{\star}(\mathbf{p}^s),\mathbf{p}^s)$ and $g_{p_i}(\mathbf{x}_i^{\star}(\mathbf{p}^{s}){,}\mathbf{p}^{s})$ cannot be recovered due to approximation and learning errors, we will use $\tilde{g}_a(\mathbf{p}^{\epsilon s}{;}\bm{\hat{\Theta}}^J)$ and $g_{p_i}(\mathbf{x}_i^{\epsilon s}{,}\mathbf{p}^{\epsilon s})$ from (\ref{agg_learning_appr}), (\ref{pros_learning_appr}) to get quality bounds for these equilibrium utilities. Instead of only learning the solution map \cite{sinha2017evolutionary}, Theorem \ref{thm:model_mismatch} also connects the approximate solution to the maximum approximation error, via Lipschitzian stability of $(\mathbf{x}_i^{\star}(\mathbf{p}^{\epsilon s}),\mathbf{p}^{\epsilon s})$ from (\ref{agg_learning_appr}), (\ref{pros_learning_appr}) on $\bm{\hat{\Theta}}^{J}$ perturbations, Assumptions \ref{ass:gpi_Lips}, \ref{ass:Second_Order_Quadratic_Growth} and Proposition \ref{prop:ga_tilde_Lips}.

\begin{assumption}\label{ass:gpi_Lips}
Assume that prosumer's $i$ best response $\mathbf{x}_i^{\star}(\mathbf{p})$, is Lipschitz continuous on $\mathbb{P}$ modulus $L_i {\geq} 0$.\frqed
\end{assumption}

\begin{proposition}\label{prop:ga_tilde_Lips}
Assume that $\|\bm{{\Theta}}\|_F \leq \theta^\mathrm{max} < \infty$, $\forall \bm{{\Theta}} \in \mathbb{\Theta}$ and that Assumption \ref{ass:PWL_basis} holds. Then, $\tilde{g}_a(\mathbf{p}{;}\bm{{\Theta}})$ is Lipschitz continuous on $\mathbb{P} \times \mathbb{\Theta}$ modulus $L_a \geq 0$, where,
\begin{equation*}
\begin{split}
    L_a = \max & \{(\lambda^\mathrm{max}{+}p^\mathrm{max}) \theta^\mathrm{max} L_{\phi} {+} \|\mathbf{d}^{\o{}} \|_2 {+} \theta^\mathrm{max} \phi^\mathrm{max}, \\
    & (\lambda^\mathrm{max} {+} p^\mathrm{max}) \phi^\mathrm{max}\}
\end{split}
\end{equation*} 
\end{proposition}
\begin{proof}
The proof is a direct consequence of the above results and thus is omitted. \frQED
\end{proof}

\begin{assumption}\label{ass:Second_Order_Quadratic_Growth}
Assume that $\exists$ $\delta>0$, such that,
\begin{equation*}
    \tilde{g}_a(\mathbf{p}{;}\bm{\Theta_c}^{\star}) \leq \max_{\mathbf{p} \in \mathbb{P}} \tilde{g}_a(\mathbf{p}{;}\bm{\Theta_c}^{\star}) - \delta \cdot \textrm{dist}(\mathbf{p}, \mathbb{M}^{s})^2,
\end{equation*}
\noindent
$\forall \mathbf{p} \in \mathbb{P} \cap \mathbb{W}$, where $\mathbb{W}$ is an open, convex neighborhood of $\mathbb{M}^{s}$ in $\mathbb{P}$. The problem $\max_{\mathbf{p} \in \mathbb{P}} \tilde{g}_a(\mathbf{p}{;}\bm{\Theta_c}^{\star})$ is bounded on $\mathbb{P}$ and $\textrm{dist}(\mathbf{p}, \mathbb{M}^{s}) {=} \min_{\mathbf{p} \in \mathbb{P} \cap \mathbb{W}} \{{\|\mathbf{p} {-} \mathbf{p}^s\|_2: \mathbf{p}^s \in \mathbb{M}^{s}} \}$ is the minimum distance between $\mathbf{p}$ and the set of optimal solutions $\mathbb{M}^{s}$ of the optimization problem $\max_{\mathbf{p} \in \mathbb{P}} \tilde{g}_a(\mathbf{p}{;}\bm{\Theta_c}^{\star})$, i.e., $\mathbb{M}^{s}$ contains the multiple solutions $\mathbf{p}^s$ of (\ref{agg_learning_inf}), which are the same solutions as those of (\ref{agg_obj}), (\ref{agg_in_con}). This standard assumption is known as the local Quadratic Growth (QG) condition.\frqed
\end{assumption}

Note that the set of optimal solutions $\mathbb{M}^{s}$ of (\ref{agg_learning_inf}) in Assumption \ref{ass:Second_Order_Quadratic_Growth}, has an unknown topology in the neighborhood $\mathbb{W}$. Since the forms of best responses and $\mathbf{d}^{\star}(\mathbf{p})$ are unknown (Remark \ref{rem:KKT_no_solve}) and without any other assumptions on the structure of $\bm{\phi}(\mathbf{p})$ and $\tilde{g}_a(\mathbf{p}{;}\bm{\Theta_c}^{\star})$, it is impossible to know if $\mathbb{M}^{s}$ contains one maximum, multiple isolated maxima or multiple non-isolated maxima. It was shown, however, that $\mathbb{M}^{s}$ is non-empty. Nevertheless, the QG condition locally holds in every case, according to \cite{shapiro1992perturbation, bonnans1995second, bonnans1995quadratic}. If $\mathbb{M}^{s}$ is singleton, the QG condition becomes $\tilde{g}_a(\mathbf{p}{;}\bm{\Theta_c}^{\star}) \leq \tilde{g}_a(\mathbf{p}^s{;}\bm{\Theta_c}^{\star}) {-} \delta \cdot \|\mathbf{p} {-} \mathbf{p}^s\|_2^2$, where $\mathbf{p}^{s}$ is locally unique \cite{shapiro1992perturbation}. Note that Assumption \ref{ass:Second_Order_Quadratic_Growth} is ensured if Second Order Sufficient Conditions hold locally.

\begin{figure*}[!ht]
    \centering
    \includegraphics[width=0.9\textwidth]{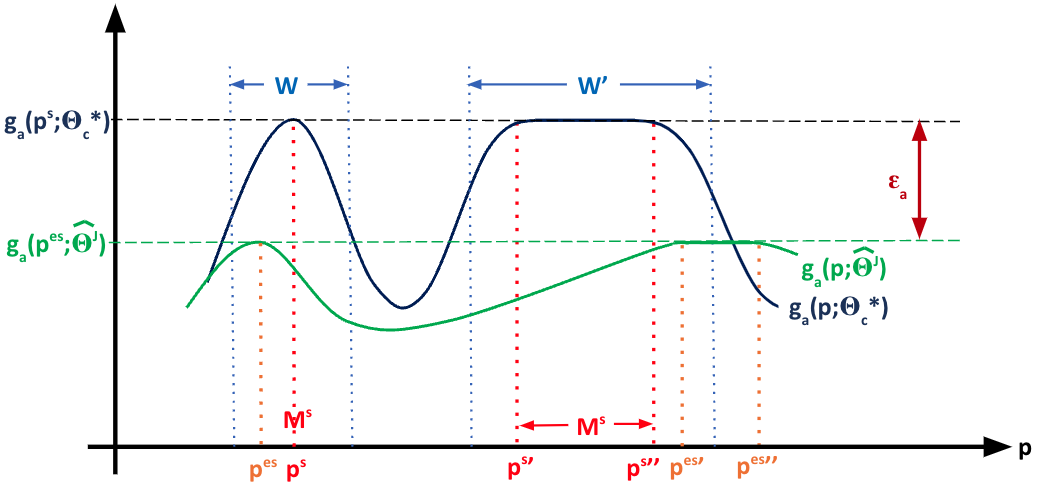}
    \caption{Original (\ref{agg_learning_inf}) and perturbed (\ref{agg_learning_appr}) DR-aggregator's problems for different parameter values for Theorem \ref{thm:model_mismatch}.}
    \label{fig:ga_ga_tilde}
\end{figure*}

\begin{theorem}\label{thm:model_mismatch}
Consider the basis $\bm{\phi}(\mathbf{p}) \in \mathbb{R}^K$, $K < \infty$, under Assumption \ref{ass:PWL_basis} and two errors in the estimation of the prosumers' sum of best responses $\mathbf{d}^{\star} \in \mathbb{R}^T$. First, a bounded approximation error, $\bm{\epsilon}_m(\mathbf{p}) \in \mathbb{R}^T$, with $\|\bm{\epsilon}_m(\mathbf{p})\|_2 \leq \epsilon_m^\mathrm{max} < \infty$, $\forall \mathbf{p} \in [0, p^\mathrm{max}]$, due to $\mathbf{d}^{\star}(\mathbf{p}) \notin \textrm{span}\{\bm{\phi}(\mathbf{p})\}$. Second, a bounded learning error $\bm{\tilde{\Theta}}^{J^{\mathsf{T}}} \bm{\phi}(\mathbf{p}) \in \mathbb{R}^T$, due to the original approximation error (Theorem \ref{thm:UUB_PE_with_wo_appr_error}). Under these errors, Algorithm \ref{algo:stackelberg_algo_online}, converges to an $\epsilon$-Stackelberg solution $(\mathbf{x}_i^{\epsilon s},\mathbf{p}^{\epsilon s})$ of the game (\ref{pros_obj}), (\ref{pros_eq_con}), (\ref{pros_in_con}), (\ref{agg_obj}), (\ref{agg_in_con}), as $J {\rightarrow} \infty$, for which the players' equilibrium utilities are bounded as follows,
\begin{equation}
    \tilde{g}_a(\mathbf{p}^{\epsilon s}{;}\bm{\hat{\Theta}}^J) {-} \epsilon_a {\leq} g_a(\mathbf{d}^{\star}(\mathbf{p}^s){,}\mathbf{p}^s) {\leq} \tilde{g}_a(\mathbf{p}^{\epsilon s}{;}\bm{\hat{\Theta}}^J) {+} \epsilon_a{,}
\end{equation}
\begin{equation}
    g_{p_i}(\mathbf{x}_i^{\epsilon s}{,}\mathbf{p}^{\epsilon s}) {-} \epsilon_{p_i} {\leq} g_{p_i}(\mathbf{x}_i^{\star}(\mathbf{p}^{s}){,}\mathbf{p}^{s}) {\leq} g_{p_i}(\mathbf{x}_i^{\epsilon s}{,}\mathbf{p}^{\epsilon s}) {+} \epsilon_{p_i},
\end{equation}
\noindent
where,
\begin{equation}
\begin{split}
    \epsilon_a & = L_a ((\delta^{-1}{+}2\delta^{{-}1/2}L_a) (\sqrt{T}(\lambda^\mathrm{max}{+}p^\mathrm{max})\\
    & (\eta \epsilon_m^\mathrm{max} L_{\phi} {+} \theta_r^\mathrm{max} L_{\partial \phi}){+} (\eta\phi^\mathrm{max}{+}1) \epsilon_m^\mathrm{max}){+}\theta_r^\mathrm{max} {+} \eta\epsilon_m^\mathrm{max}),
\end{split}
\label{eq:eps_a}
\end{equation}
\begin{equation}
\begin{split}
    \epsilon_{p_i} & = L_{p_i} (L_i {+} 1) (\delta^{-1}{+}2\delta^{{-}1/2}L_a) (\sqrt{T}(\lambda^\mathrm{max}{+}p^\mathrm{max}) \qquad \quad\\
    & (\eta \epsilon_m^\mathrm{max} L_{\phi} {+} \theta_r^\mathrm{max} L_{\partial \phi}){+} (\eta\phi^\mathrm{max}{+}1) \epsilon_m^\mathrm{max}),
\end{split}
\label{eq:eps_p_i}
\end{equation}
and $\tilde{g}_a(\mathbf{p}^{\epsilon s}{;}\bm{\hat{\Theta}}^J)$, $g_{p_i}(\mathbf{x}_i^{\epsilon s},\mathbf{p}^{\epsilon s})$ are the objectives of (\ref{agg_learning_appr}), (\ref{pros_learning_appr}) for which Assumptions \ref{ass:gpi_Lips}, \ref{ass:Second_Order_Quadratic_Growth} and Proposition \ref{prop:ga_tilde_Lips} hold. 
\end{theorem}
\begin{proof}
Since, the dynamics of $\bm{\hat{\Theta}}^J$ in (\ref{RLS_theta}), are Uniformly Ultimately Bounded (UUB) (see proof of Theorem \ref{thm:UUB_PE_with_wo_appr_error}), it holds by the definition of UUB \cite{lewis2020neural} that  $\|\bm{\hat{\Theta}}^J\|_F {\leq} \theta^\mathrm{max}$, $\forall J {\geq} J_\mathrm{min}$, with $0 {<} J_\mathrm{min} {<} \infty$, $\theta^\mathrm{max} {>} 0$, given that $\|\bm{\hat{\Theta}}^0\|_F$ was bounded upon initialization. Under Proposition \ref{prop:ga_tilde_Lips}, it holds that $\tilde{g}_a(\mathbf{p}{;}\bm{{\Theta}})$ is Lipschitz continuous on $\mathbb{P} \times \mathbb{\Theta}$ and hence,
\begin{equation*}
 \begin{split}
    & d'(\tilde{g}_a(\mathbf{p}{;}\bm{\Theta}) {,} \tilde{g}_a(\mathbf{p}'{;}\bm{\Theta}'))
    \leq L_a d((\mathbf{p}{,} \bm{\Theta}){,}(\mathbf{p}'{,}\bm{\Theta}')){,}\\
    & \forall (\mathbf{p}{,} \bm{\Theta}){,}(\mathbf{p}'{,}\bm{\Theta}') \in \mathbb{P} \times \mathbb{\Theta},
\end{split}   
\end{equation*}
\noindent
where $(d',\mathbb{R})$ and $(d,\mathbb{P} \times \mathbb{\Theta})$ are metric spaces and $\tilde{g}_a: \mathbb{P} \times \mathbb{\Theta} \rightarrow \mathbb{R}$. Now, let Assumption \ref{ass:Second_Order_Quadratic_Growth} hold on an open, convex neighborhood $\mathbb{W}$ of $\mathbb{M}^{s}$. Then, $\tilde{g}_a$ will also be locally Lipschitz continuous on $\mathbb{P} \cap \mathbb{W}$ and for a properly defined distance of point-to-set, that satisfies the distance properties on metric spaces \cite{debnath2005introduction}, similar to the one in Assumption \ref{ass:Second_Order_Quadratic_Growth},
\begin{equation*}
\begin{split}
    & |\tilde{g}_a(\mathbf{p}^s{;}\bm{\Theta_c}^{\star}) {-} \tilde{g}_a(\mathbf{p}'{;}\bm{\Theta}')|
    \leq L_a (\textrm{dist}(\mathbf{p}', \mathbb{M}^{s})
    {+} \|\bm{\Theta_c}^{\star}{-}\bm{\Theta}'\|_2){,}\\
    & \forall (\mathbf{p}^s{,} \bm{\Theta_c}^{\star}){,}(\mathbf{p}'{,}\bm{\Theta}') \in \mathbb{P} \cap \mathbb{W} \times \mathbb{\Theta}.
\end{split}
\end{equation*}
\noindent
Note that $\bm{\Theta_c}^{\star}$ is singleton from (\ref{LS_Problem}). Consider now, that in the neighborhood $\mathbb{P} \cap \mathbb{W}$ where QG condition holds for $\tilde{g}_a(\mathbf{p}{;}\bm{\Theta_c}^{\star})$, the function $\tilde{g}_a(\mathbf{p}{;}\bm{\hat{\Theta}}^J)$ has at least one maximum, i.e., there is at least one $\mathbf{p}^{\epsilon s} \in \mathbb{P} \cap \mathbb{W}$ which solves (\ref{agg_learning_appr}) to optimality. A similar assumption was made in section 3 of \cite{shapiro1992perturbation}. Then, the previous definition for $(\mathbf{p}'{,}\bm{\Theta}'){=} (\mathbf{p}^{\epsilon s}{,}\bm{\hat{\Theta}}^J)$ is,
\begin{equation*}
\begin{split}
    & |\tilde{g}_a(\mathbf{p}^s{;}\bm{\Theta_c}^{\star}) {-} \tilde{g}_a(\mathbf{p}^{\epsilon s}{;}\bm{\hat{\Theta}}^J)|
    {\leq} L_a (\textrm{dist}(\mathbf{p}^{\epsilon s}, \mathbb{M}^{s})
    {+} \|\bm{\Theta_c}^{\star}{-}\bm{\hat{\Theta}}^J\|_2){,}\\
\end{split}
\end{equation*}
\noindent
and since $g_a(\mathbf{d}^{\star}(\mathbf{p}^s), \mathbf{p}^s) {=} \tilde{g}_a(\mathbf{p}^s;\bm{\Theta_c}^{\star})$,
\begin{equation*}
\begin{split}
    |g_a & (\mathbf{d}^{\star}(\mathbf{p}^s){,} \mathbf{p}^s) {-} \tilde{g}_a(\mathbf{p}^{\epsilon s}{;}\bm{\hat{\Theta}}^J)| {\leq} \\
    & {\leq} L_a (\textrm{dist}(\mathbf{p}^{\epsilon s}{,} \mathbb{M}^{s})
    {+} \|\bm{\Theta_c}^{\star}{-}\bm{\hat{\Theta}}^J\|_2){,}\\
    & {\leq} L_a (\textrm{dist}(\mathbf{p}^{\epsilon s}{,} \mathbb{M}^{s})
    {+} \|\bm{\Theta_c}^{\star}{-}\bm{\hat{\Theta}}^J\|_F){,}\\
\end{split}
\end{equation*}
\noindent
where $\mathbf{d}^{\star}(\mathbf{p}^s){=}\bm{\Theta_c}^{{\star}^{\mathsf{T}}} \bm{\phi}(\mathbf{p}^s)$ and $\|A\|_2 {\leq} \|A\|_F$ for any $A \in \mathbb{R}^{K \times T}$ (\textsection 5.2 in \cite{meyer2000matrix}). From Lemma \ref{lemma:approx_error} and Theorem \ref{thm:UUB_PE_with_wo_appr_error},
\begin{equation}
\begin{split}
    |g_a(&\mathbf{d}^{\star}(\mathbf{p}^s),\mathbf{p}^s) {-} \tilde{g}_a(\mathbf{p}^{\epsilon s}{;}\bm{\hat{\Theta}}^J)| \leq \\
    \leq & L_a (\textrm{dist}(\mathbf{p}^{\epsilon s}, \mathbb{M}^{s})
    {+} \|\bm{\Theta_c}^{\star}{-}\bm{\hat{\Theta}}^J\|_F)\\
    \leq & L_a (\textrm{dist}(\mathbf{p}^{\epsilon s},\mathbb{M}^{s}) {+} \|\bm{\Theta_c}^{\star}{-} \mathbf{\Theta^{\star}}\|_F {+} \|\mathbf{\Theta^{\star}}{-}\bm{\hat{\Theta}}^J\|_F \\
    = & L_a (\textrm{dist}(\mathbf{p}^{\epsilon s},\mathbb{M}^{s}) {+} \|\bm{\Theta_c}^{\star}{-} \mathbf{\Theta^{\star}}\|_F {+} \|\mathbf{\tilde{\Theta}}^J\|_F),\\
    & \forall (\mathbf{p}^s{,} \bm{\Theta_c}^{\star}){,}(\mathbf{p}^{\epsilon s}{,}\bm{\hat{\Theta}}^J) \in \mathbb{P} \cap \mathbb{W} \times \mathbb{\Theta}.
\end{split}
\label{ga_tilde_Lip_bounds}
\end{equation}
\noindent
Upper bounds of (\ref{ga_tilde_Lip_bounds}) will be extracted next, as functions of the approximation error. For the prosumers, one can prove, similarly to Proposition \ref{prop:ga_tilde_Lips}, that $g_{p_i}(\mathbf{x}_i{,}\mathbf{p})$ is Lipschitz continuous on $\mathbb{X}_i \times \mathbb{P}$ modulus $L_{p_i} {\geq} 0$, since $0{<}u_i{<}\infty$ from section \ref{sec:prob_form} and $\|\mathbf{x}_i^{\o{}}\|_2{<}\infty$ from Assumption \ref{ass:pros_sched_demand}. Therefore,
\begin{equation*}
\begin{split}
    & d'(g_{p_i}(\mathbf{x}_i,\mathbf{p}) {,} g_{p_i}(\mathbf{x}_i',\mathbf{p}'))
    \leq L_{p_i} d((\mathbf{x}_i{,}\mathbf{p}){,}(\mathbf{x}_i'{,}\mathbf{p}')){,}\\
    & \forall (\mathbf{x}_i{,}\mathbf{p}), (\mathbf{x}_i'{,}\mathbf{p}') \in \mathbb{X}_i \times \mathbb{P},\\
\end{split}
\end{equation*}
\noindent
and locally on $\mathbb{X}_i \times \mathbb{P} \cap \mathbb{W}$, as before, for some $\mathbf{p}^{\epsilon s} \in \mathbb{P} \cap \mathbb{W}$,
\begin{equation*}
    |g_{p_i}(\mathbf{x}_i^{s}{,}\mathbf{p}^{s}) {-}  g_{p_i}(\mathbf{x}_i^{\epsilon s}{,}\mathbf{p}^{\epsilon s})| 
    {\leq} L_{p_i} (\|\mathbf{x}_i^s {-} \mathbf{x}_i^{\epsilon s}\|_2 {+}\textrm{dist}(\mathbf{p}^{\epsilon s}{,} \mathbb{M}^{s})){,}
\end{equation*}
\noindent
or under Assumption \ref{ass:gpi_Lips},
\begin{equation}
\begin{split}
    |g_{p_i}(&\mathbf{x}_i^{\star}(\mathbf{p}^{s}),\mathbf{p}^{s}) {-}  g_{p_i}(\mathbf{x}_i^{\epsilon s},\mathbf{p}^{\epsilon s})| \leq  \\
    {\leq} & L_{p_i} (\|\mathbf{x}_i^{\star}(\mathbf{p}^s) {-} \mathbf{x}_i^{\star}(\mathbf{p}^{\epsilon s})\|_2 {+} \textrm{dist}(\mathbf{p}^{\epsilon s}{,} \mathbb{M}^{s})) \\
    {\leq} & L_{p_i} (L_i \cdot \textrm{dist}(\mathbf{p}^{\epsilon s}{,} \mathbb{M}^{s}) {+} \textrm{dist}(\mathbf{p}^{\epsilon s}{,} \mathbb{M}^{s})) \\
    {=} & L_{p_i} (L_i {+} 1) \textrm{dist}(\mathbf{p}^{\epsilon s}{,} \mathbb{M}^{s}), \\
    & \forall (\mathbf{x}_i^{\star}(\mathbf{p}^s){,}\mathbf{p}^s){,}(\mathbf{x}_i^{\star}(\mathbf{p}^{\epsilon s}){,}\mathbf{p}^{\epsilon s}) \in \mathbb{X}_i \times \mathbb{P} \cap \mathbb{W}.
\label{gpi_tilde_Lip_bounds}
\end{split}
\end{equation}
\noindent
The stability of optimization problem solutions to parameter perturbation and upper bounds for $\textrm{dist}(\mathbf{p}^{\epsilon s}{,} \mathbb{M}^{s})$ were provided in \cite{shapiro1992perturbation}. To use Lemma 2.1 and section 3 of \cite{shapiro1992perturbation}, we must show that the difference between objective (\ref{agg_learning_appr}) and objective (\ref{agg_learning_inf}) is G\^{a}teaux differentiable on $\mathbb{W}$. It holds,
\begin{equation*}
\begin{split}
    \mathbf{d}^{\star}(\mathbf{p}) & = \bm{\Theta_c}^{{\star}^{\mathsf{T}}} \bm{\phi}(\mathbf{p}) = \mathbf{\Theta^{\star}}^{\mathsf{T}} \bm{\phi}(\mathbf{p}) + \bm{\Theta_r}^{\mathsf{T}} \bm{\phi}(\mathbf{p}) \\
    & =  \bm{\tilde{\Theta}}^{J^{\mathsf{T}}} \bm{\phi}(\mathbf{p}) + \bm{\hat{\Theta}}^{J^{\mathsf{T}}} \bm{\phi}(\mathbf{p}) + \bm{\Theta_r}^{\mathsf{T}} \bm{\phi}(\mathbf{p}) \\
    & = \bm{\tilde{\Theta}}^{J^{\mathsf{T}}} \bm{\phi}(\mathbf{p}) + \bm{\hat{\Theta}}^{J^{\mathsf{T}}} \bm{\phi}(\mathbf{p}) + \bm{\epsilon}_m(\mathbf{p}),
\end{split}
\label{d_star_theta_decomposition}
\end{equation*}
\begin{equation*}
\begin{split}
    \tilde{g}_a(\mathbf{p}{;}\bm{\hat{\Theta}}^J) {=} & (\bm{\lambda} {-} \mathbf{p})^\mathsf{T} (\mathbf{d}^{\o{}}{-}\bm{\hat{\Theta}}^{J^{\mathsf{T}}} \bm{\phi}(\mathbf{p}))\\
    {=} & (\bm{\lambda} {-} \mathbf{p})^\mathsf{T} (\mathbf{d}^{\o{}}{-}\bm{\hat{\Theta}}^{J^{\mathsf{T}}} \bm{\phi}(\mathbf{p}) {-} \bm{\Theta_c}^{{\star}^{\mathsf{T}}} \bm{\phi}(\mathbf{p}) {+} \bm{\Theta_c}^{{\star}^{\mathsf{T}}} \bm{\phi}(\mathbf{p}))\\
    {=} & (\bm{\lambda} {-} \mathbf{p})^\mathsf{T} (\mathbf{d}^{\o{}}{-}\bm{\Theta_c}^{{\star}^{\mathsf{T}}} \bm{\phi}(\mathbf{p}))
    {+} (\bm{\lambda} {-} \mathbf{p})^\mathsf{T} (\bm{\tilde{\Theta}}^J {+} \bm{\Theta_r})^{\mathsf{T}} \bm{\phi}(\mathbf{p})\\
    {=} & \tilde{g}_a(\mathbf{p}{;}\bm{\Theta_c}^{\star}) {+} h(\mathbf{p}{;}\bm{\tilde{\Theta}}^J{+}\bm{\Theta_r}),
\end{split}
\end{equation*}
\noindent
where if $\bm{\tilde{\Theta}}^J{+}\bm{\Theta_r}{=}0$, then $\bm{\Theta}^{\star}{-}\bm{\hat{\Theta}}^J{+}\bm{\Theta_c}^{\star}{-}\bm{\Theta}^{\star}{=}0$, or $\bm{\hat{\Theta}}^J{=}\bm{\Theta_c}^{\star}$. In that case $\tilde{g}_a(\mathbf{p}{;}\bm{\hat{\Theta}}^J){=}\tilde{g}_a(\mathbf{p}{;}\bm{\Theta_c}^{\star})$ and hence $h(\mathbf{p}{;}\bm{0}){=}0$. The derivative of the second term,
\begin{equation*}
\begin{split}
    \frac{\partial h(\mathbf{p}{;}\bm{\tilde{\Theta}}^J{+}\bm{\Theta_r})}{\partial \mathbf{p}}
    {=} & (\bm{\lambda}{-}\mathbf{p})^\mathsf{T} (\bm{\tilde{\Theta}}^J {+} \bm{\Theta_r})^\mathsf{T} 
    \frac{\partial \bm{\phi}(\mathbf{p})}{\partial \mathbf{p}}\\
    &{-} \bm{\phi}(\mathbf{p})^\mathsf{T} (\bm{\tilde{\Theta}}^J {+} \bm{\Theta_r}){,}
\end{split}
\end{equation*}
\noindent
is Lipschitz continuous under Assumption \ref{ass:PWL_basis}, as a composite of Lipschitz continuous functions (proof as in Proposition \ref{prop:ga_tilde_Lips}). Since $h(\mathbf{p}{;}\bm{\tilde{\Theta}}^J{+}\bm{\Theta_r})$ has Lipschitz continuous derivatives on $\mathbb{W}$, it is also Fr\'{e}chet differentiable on $\mathbb{W}$ and hence G\^{a}teaux differentiable on $\mathbb{W}$ (example 4.33 in \cite{bonnans2013perturbation}).

According to Lemma 2.1, section 3 in \cite{shapiro1992perturbation} and under Assumption \ref{ass:Second_Order_Quadratic_Growth}, if there is at least one $\mathbf{p}^{\epsilon s} \in \mathbb{P} \cap \mathbb{W}$ and if $h(\mathbf{p}{;}\bm{\tilde{\Theta}}^J{+}\bm{\Theta_r})$ is G\^{a}teaux differentiable on $\mathbb{W}$ (proved above), then the following holds from Mean Value Theorem,
\begin{equation}
    \textrm{dist}(\mathbf{p}^{\epsilon s}{,} \mathbb{M}^{s}) {\leq} (\delta^{-1}{+}\delta^{{-}1/2}L_h) \max_{\mathbf{p} \in \mathbb{W}} \Big{\|}\frac{\partial h(\mathbf{p}{;}\bm{\tilde{\Theta}}^J{+}\bm{\Theta_r})}{\partial \mathbf{p}}\Big{\|}_2^{\star}{,}
\label{shapiro_upper_bound}
\end{equation}
\noindent
where $\delta{>}0$ is the constant in Assumption \ref{ass:Second_Order_Quadratic_Growth} and $L_h{=}2L_a {\geq} 0$ is the Lipschitz constant of $h(\mathbf{p}{;}\bm{\tilde{\Theta}}^J{+}\bm{\Theta_r})$, where the proof is omitted for brevity. 
Now, since the result of Lemma 2.1 of \cite{shapiro1992perturbation} also holds for maximization problems and since the dual norm of a Euclidean norm is the norm itself, (\ref{shapiro_upper_bound}) becomes,
\begin{equation*}
\begin{split}
    \max_{\mathbf{p} \in \mathbb{W}} & \Big{\|}\frac{\partial h(\mathbf{p}{;}\bm{\tilde{\Theta}}^J{+}\bm{\Theta_r})}{\partial \mathbf{p}}\Big{\|}_2^{\star}
    {=} \max_{\mathbf{p} \in \mathbb{W}} \Big{\|}\frac{\partial h(\mathbf{p}{;}\bm{\tilde{\Theta}}^J{+}\bm{\Theta_r})}{\partial \mathbf{p}}\Big{\|}_2 \\ 
    {=} & \max_{\mathbf{p} \in \mathbb{W}} \Big{\|}  (\bm{\lambda}{-}\mathbf{p})^\mathsf{T} (\bm{\tilde{\Theta}}^J {+} \bm{\Theta_r})^\mathsf{T} 
    \frac{\partial \bm{\phi}(\mathbf{p})}{\partial \mathbf{p}} {-} \bm{\phi}(\mathbf{p})^\mathsf{T} (\bm{\tilde{\Theta}}^J {+} \bm{\Theta_r}) \Big{\|}_2 \\ 
    {\leq} & \max_{\mathbf{p} \in \mathbb{W}} ((\|\bm{\lambda}\|_2 {+} \|\mathbf{p}\|_2) (\|\bm{\tilde{\Theta}}^J\|_2 \Big{\|} \frac{\partial \bm{\phi}(\mathbf{p})}{\partial \mathbf{p}} \Big{\|}_2 {+}\|\bm{\Theta_r}^\mathsf{T} \frac{\partial \bm{\phi}(\mathbf{p})}{\partial \mathbf{p}} \|_2)  \\
    & \qquad \qquad \quad {+} \|\bm{\phi}(\mathbf{p})\|_2 \|\bm{\tilde{\Theta}}^J\|_2{+}\| \bm{\phi}(\mathbf{p})^\mathsf{T} \bm{\Theta_r}\|_2)\\
    {\leq} & \max_{\mathbf{p} \in \mathbb{W}} ((\|\bm{\lambda}\|_2 {+} \|\mathbf{p}\|_2) (\|\bm{\tilde{\Theta}}^J\|_F \Big{\|} \frac{\partial \bm{\phi}(\mathbf{p})}{\partial \mathbf{p}} \Big{\|}_2 {+}\|\bm{\Theta_r}^\mathsf{T} \frac{\partial \bm{\phi}(\mathbf{p})}{\partial \mathbf{p}} \|_2)  \\
    & \qquad \qquad \quad {+} \|\bm{\phi}(\mathbf{p})\|_2 \|\bm{\tilde{\Theta}}^J\|_F{+}\| \bm{\phi}(\mathbf{p})^\mathsf{T} \bm{\Theta_r}\|_2)\\
    {\leq} & \max_{\mathbf{p} \in \mathbb{W}} ((\sqrt{T} \lambda^\mathrm{max} {+} \|\mathbf{p}\|_2) (\eta \epsilon_m^\mathrm{max} L_{\phi} {+}\|\bm{\Theta_r}^\mathsf{T} \frac{\partial \bm{\phi}(\mathbf{p})}{\partial \mathbf{p}} \|_2) \\
    & \qquad \qquad \quad {+} \phi^\mathrm{max} \eta \epsilon_m^\mathrm{max} {+} \epsilon_m^\mathrm{max})\\
    {\leq} & \sqrt{T}(\lambda^\mathrm{max}{+}p^\mathrm{max}) (\eta \epsilon_m^\mathrm{max} L_{\phi} {+} \theta_r^\mathrm{max} L_{\partial \phi}) \\
    & \qquad \qquad \quad {+} (\eta\phi^\mathrm{max}{+}1) \epsilon_m^\mathrm{max},
\end{split}
\end{equation*} 
\noindent
where $\mathbf{p} \in \mathbb{R}^T$, $\bm{\lambda} \in \mathbb{R}^T$ and the upper bounds were selected from inequality (\ref{agg_in_con}), from section \ref{sec:prob_form}, from Assumptions \ref{ass:pros_sched_demand}, \ref{ass:PWL_basis} and \ref{ass:bounded_theta_r}, from Lemma \ref{lemma:approx_error} and from Theorem \ref{thm:UUB_PE_with_wo_appr_error}. Therefore,
\begin{equation}
\begin{split}
    \textrm{dist}&(\mathbf{p}^{\epsilon s}{,} \mathbb{M}^{s})  {\leq} (\delta^{-1}{+}2\delta^{{-}1/2}L_a) (\sqrt{T}(\lambda^\mathrm{max}{+}p^\mathrm{max})\\
    & (\eta \epsilon_m^\mathrm{max} L_{\phi} {+} \theta_r^\mathrm{max} L_{\partial \phi}){+} (\eta\phi^\mathrm{max}{+}1) \epsilon_m^\mathrm{max}).
\end{split}
\label{p_upper_bound}
\end{equation}
\noindent
The upper bound estimated in (\ref{p_upper_bound}) and the upper bounds from Lemma \ref{lemma:approx_error} and Theorem \ref{thm:UUB_PE_with_wo_appr_error}, are used to bound (\ref{ga_tilde_Lip_bounds}) and (\ref{gpi_tilde_Lip_bounds}) and finalize $\epsilon_a$ and $\epsilon_{p_i}$. Specifically, from (\ref{ga_tilde_Lip_bounds}),
\begin{equation*}
\begin{split}
    |g_a&(\mathbf{d}^{\star}(\mathbf{p}^s),\mathbf{p}^s) {-} \tilde{g}_a(\mathbf{p}^{\epsilon s}{;}\bm{\hat{\Theta}}^J)| \leq \\
    \leq & L_a (\textrm{dist}(\mathbf{p}^{\epsilon s},\mathbb{M}^{s}) {+} \|\bm{\Theta_c}^{\star}{-} \mathbf{\Theta^{\star}}\|_F {+} \|\mathbf{\tilde{\Theta}}^J\|_F)\\
    \leq & L_a ((\delta^{-1}{+}2\delta^{{-}1/2}L_a) (\sqrt{T}(\lambda^\mathrm{max}{+}p^\mathrm{max})\\
    & (\eta \epsilon_m^\mathrm{max} L_{\phi} {+} \theta_r^\mathrm{max} L_{\partial \phi}){+} (\eta\phi^\mathrm{max}{+}1) \epsilon_m^\mathrm{max}){+}\theta_r^\mathrm{max} {+} \eta\epsilon_m^\mathrm{max}),\\
    & \forall (\mathbf{p}^s{,} \bm{\Theta_c}^{\star}){,}(\mathbf{p}^{\epsilon s}{,}\bm{\hat{\Theta}}^J) \in \mathbb{P} \cap \mathbb{W} \times \mathbb{\Theta},
\end{split}
\end{equation*}
\noindent
from where $\epsilon_a$ in (\ref{eq:eps_a}) occurs.
\noindent
Equivalently, from (\ref{gpi_tilde_Lip_bounds}),
\begin{equation*}
\begin{split}
    |g_{p_i}(&\mathbf{x}_i^{\star}(\mathbf{p}^{s}),\mathbf{p}^{s}) {-}  g_{p_i}(\mathbf{x}_i^{\epsilon s},\mathbf{p}^{\epsilon s})| \leq  \\
    \leq & L_{p_i} (L_i {+} 1) \textrm{dist}(\mathbf{p}^{\epsilon s}{,} \mathbb{M}^{s}) \\
    \leq & L_{p_i} (L_i {+} 1) (\delta^{-1}{+}2\delta^{{-}1/2}L_a) (\sqrt{T}(\lambda^\mathrm{max}{+}p^\mathrm{max}) \\
    & (\eta \epsilon_m^\mathrm{max} L_{\phi} {+} \theta_r^\mathrm{max} L_{\partial \phi}){+} (\eta\phi^\mathrm{max}{+}1) \epsilon_m^\mathrm{max}),\\
    & \forall (\mathbf{x}_i^{\star}(\mathbf{p}^s){,}\mathbf{p}^s){,}(\mathbf{x}_i^{\star}(\mathbf{p}^{\epsilon s}){,}\mathbf{p}^{\epsilon s}) \in \mathbb{X}_i \times \mathbb{P} \cap \mathbb{W},
\end{split}
\end{equation*}
\noindent
from where $\epsilon_{p_i}$ in (\ref{eq:eps_p_i}) occurs. By improving the upper bound (\ref{p_upper_bound}), the utility bounds (Theorem \ref{thm:model_mismatch}) become tighter and a better $\epsilon$-Stackelberg solution is attained. In fact, as the basis dimension $K {\rightarrow} \infty$, the bound $\theta_r^\mathrm{max} {\rightarrow} 0$ and $\epsilon_m^\mathrm{max} {\rightarrow} 0$. In the absence of an approximation error, i.e., $\epsilon_m^\mathrm{max}{=}0$ and $\theta_r^\mathrm{max}{=}0$, the learning error $\lim_{j {\rightarrow} \infty} \|\mathbf{\tilde \Theta}^{j}\|_F {=} 0$, as shown in Theorem \ref{thm:UUB_PE_with_wo_appr_error}. Therefore, $\textrm{dist}(\mathbf{p}^{\epsilon s}{,} \mathbb{M}^{s}){=}0$ from (\ref{p_upper_bound}) and the exact equilibrium is recovered since $\epsilon_a{=}\epsilon_{p_i}{=}0$. In that case, $g_a(\mathbf{d}^{\star}(\mathbf{p}^s),\mathbf{p}^s){=}\tilde{g}_a(\mathbf{p}^{\epsilon s}{;}\bm{\hat{\Theta}}^J)$ and $g_{p_i}(\mathbf{x}_i^{\star}(\mathbf{p}^{s}),\mathbf{p}^{s}){=}g_{p_i}(\mathbf{x}_i^{\epsilon s},\mathbf{p}^{\epsilon s})$. 
\frQED
\end{proof}

\begin{remark}\label{rem:scalability}
Similar behavior on the basis representation improvement, was found by \cite{vamvoudakis2010online}
Scalability: The sample complexity of Algorithm \ref{algo:stackelberg_algo_online} is $O(J)$, due to the single for loop over the samples $J$, and no internal dimensions depending on $J$. Moreover, as $N {\rightarrow} \infty$, the dimensions in problem (\ref{agg_learning_appr}) and laws (\ref{RLS_theta}), (\ref{RLS_epsilon}), (\ref{RLS_pi}) are not affected. Only the summation $\mathbf{d}^{j{\star}} = \sum_{i=1}^N \mathbf{x}_i^{j\star}$, depends on $N$, indicating a linear $O(N)$ algorithmic complexity. The $N$ prosumer problems are solved simultaneously in a decentralized sense, therefore not affecting complexity. Both complexities are verified experimentally. This scalability allows for realistic applications with thousands of prosumers. \frqed
\end{remark} 

\begin{remark}\label{rem:privacy}
Privacy: The cumulative learning of the prosumers' best responses, removes the need for the communication of any problem-specific objectives, constraints, parameters, preferences, PI data measurements, etc. Similarly, the aggregator only shares random price signals and a final optimal price signal, hiding her formulation-specific details.
\frqed
\end{remark} 

\section{Experimental Results}\label{sec:exp_res}
Algorithm \ref{algo:stackelberg_algo_online} was utilized every day on hourly demand data $\mathbf{x}_{i}^{\o{}}$ from 71 buildings (prosumers) of the University of California, Davis (UC-Davis) \cite{UCDavis_Ceed_Online_Data_Tool} and on day-ahead Locational Marginal Prices (LMP) $\bm{\lambda}$ for the UC-Davis node (DAVIS-1-N030) \cite{CAISO_UC_DAVIS_NODE} of CAISO. The algorithm was deployed until $\bm{\hat{\Theta}}^j$ did not significantly change ($\sim$100 samples). No constraints were violated. The following observations are made.
\begin{enumerate}
    \item For $a{=}1$, the new total demand $\bm{\hat{\Theta}}^{J^\mathsf{T}} \bm{\phi}(\mathbf{p}^{\epsilon s})$ is restricted by inequality (\ref{pros_in_con}) and therefore, the DR-aggregator can only sell energy and not buy (Figure \ref{fig:a1_Q001}). For higher values of $a$, purchasing is also allowed. For example, for $a{=}2$ and same $Q_i{=}0.01 {\cdot} W_i$ (Figure  \ref{fig:a2_Q001}), the prosumers are more flexible to move more demand into the highly compensated hours. Therefore, all prosumers and the DR-aggregator earn more.
    \item The discrepancy between the DR-aggregator's estimated $\tilde{g}_a(\mathbf{p}^{\epsilon s}{;}\bm{\hat{\Theta}}^J)$ (green) and best prosumer response $g_a(\mathbf{d}^{\star}(\mathbf{p}^\textrm{es}),\mathbf{p}^{\epsilon s})$ (blue) rewards, is due to approximation error (Figures \ref{fig:a1_Q001}, \ref{fig:a2_Q001}, and \ref{fig:a2_Q0}). This discrepancy is reduced with more samples (Figure \ref{fig:a2_Q001_samples}). For low samples, $\tilde{g}_a(\mathbf{p}^{\epsilon s}{;}\bm{\hat{\Theta}}^J) {\leq} 0$, but for more it is positive. In low-sample solutions, i.e., 20, demand $\bm{\hat{\Theta}}^{J^\mathsf{T}} \bm{\phi}(\mathbf{p}^{\epsilon s})$ shows extreme ramping that reduces grid's resilience.
    \item The prosumer's average monthly utility is independent of their number, as there is no competition between them (Figure \ref{fig:a2_Q001_prosumers}). The DR-aggregator's utility increases if the total volume has increased. 
    \item For constant $a{=}2$ and increasing $Q_i$ (Figures \ref{fig:a2_Q001} and \ref{fig:a2_Q0}), the prosumers' and DR-aggregator's utility increase while total demand is reduced. 
    \item For $Q_i {=} 0$ (Figure \ref{fig:a2_Q0}), the prosumers move demand to cheaper hours (energy arbitrage) and still make a profit although their total demand remains the same. The aggregator makes a profit from purchasing and selling energy (bidirectional transactions) to the market. Cases with $Q_i {<} 0$ are also supported by this framework but omitted from the results for brevity.
    \item The DR-aggregator's earnings are higher than the prosumers' because of the volume of energy she moves multiplied by the price difference (price arbitrage). The DR-aggregator is also responsible for market-related fees and operational costs, not mentioned in this paper.
    \item By reducing demand ($Q_i {\neq} 0$), CO$_2$ emissions are saved, i.e., for $a{=}2$, 74864 KWh and 88075 CO$_2$ lbs were saved ($Q_i{=}0.01 {\cdot} W_i$ case), and 748645 KWh and 880759 CO$_2$ lbs were saved ($Q_i{=}0.1 {\cdot} W_i$ case), based on the 0.85 CO$_2$ lbs/KWh EIA's estimation \cite{EIA_CO2}.
    \item In this price-driven framework, demand is shifted away from expensive LMP hours, such as sunset hours. The sunset ramping can be alleviated with the large-scale deployment of DR-aggregators (Figure \ref{fig:duck_alleviation}, scalability).
    \item Linear sample $O(J)$ and prosumer $O(N)$ time complexities are demonstrated in Figures \ref{fig:sample_complex} and \ref{fig:pros_complex}.
\end{enumerate}

\begin{figure*}[!ht]
    \centering
    \subfloat[\centering Daily player utilities averaged over entire May 2022.]{{\includegraphics[width=8.0cm]{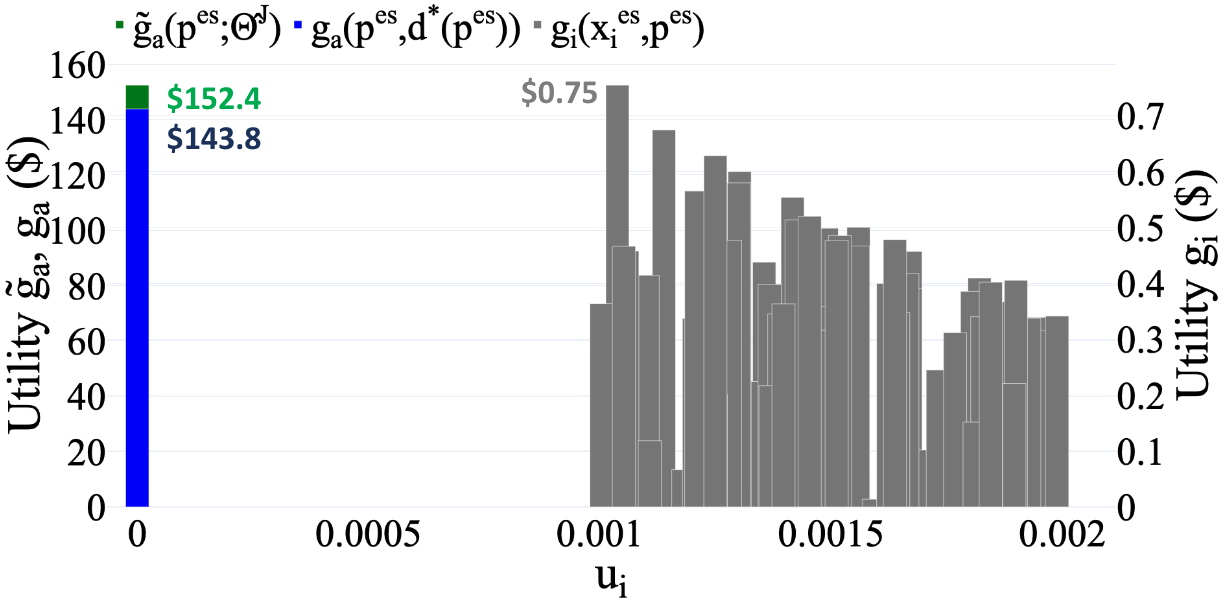} }}
    \qquad
    \subfloat[\centering Prices and demands.]{{\includegraphics[width=8.5cm]{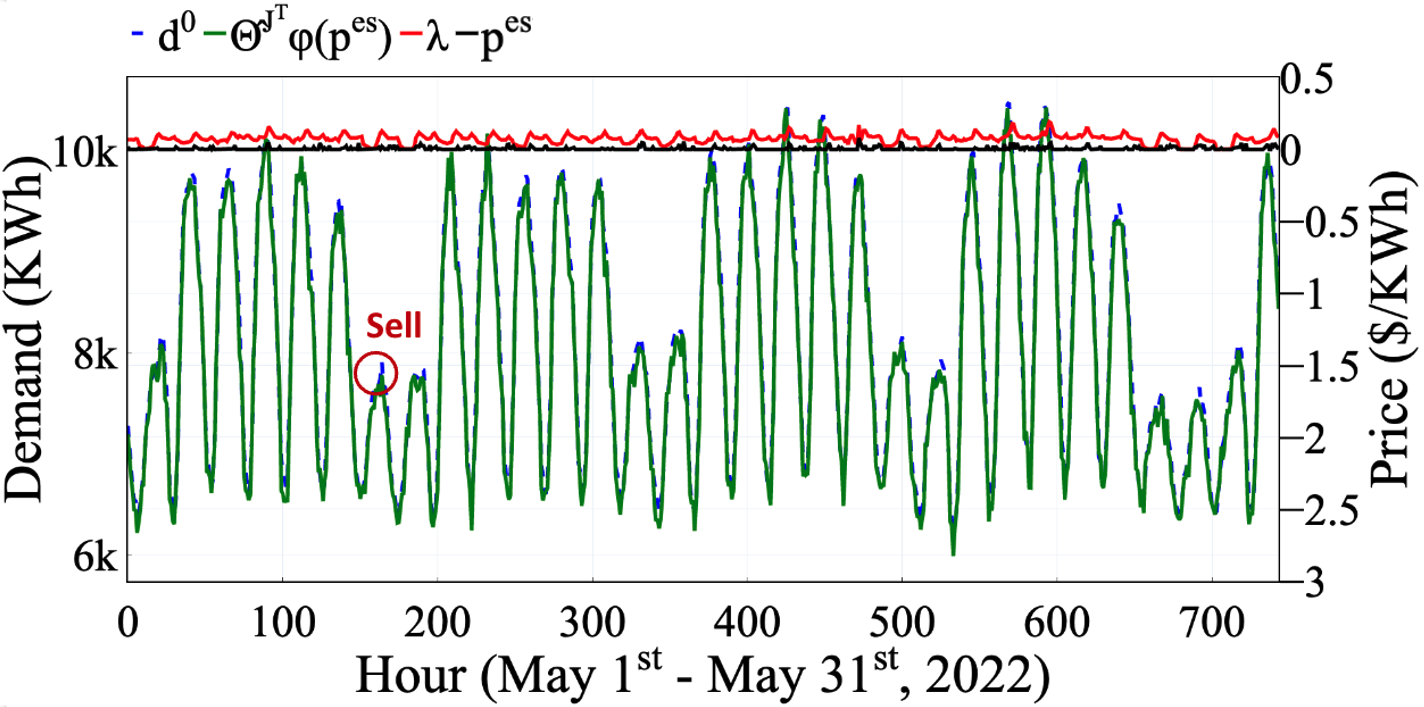} }}
    \caption{Results for $a=1$ and $Q_i=0.01\cdot W_{i}$, for May 1\ts{st}-31\ts{st}, 2022.}
    \label{fig:a1_Q001}
\end{figure*}    
\begin{figure*}[!ht]
    \centering
    \subfloat[\centering Daily player utilities averaged over entire May 2022.]{{\includegraphics[width=8.0cm]{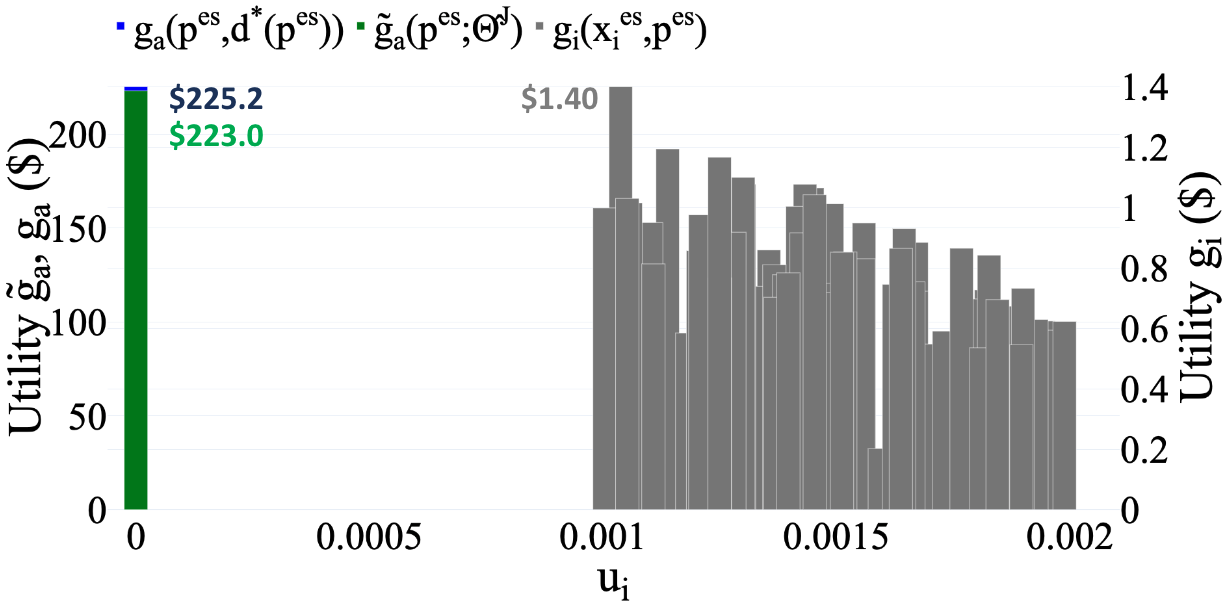} }}
    \qquad
    \subfloat[\centering Prices and demands.]{{\includegraphics[width=8.5cm]{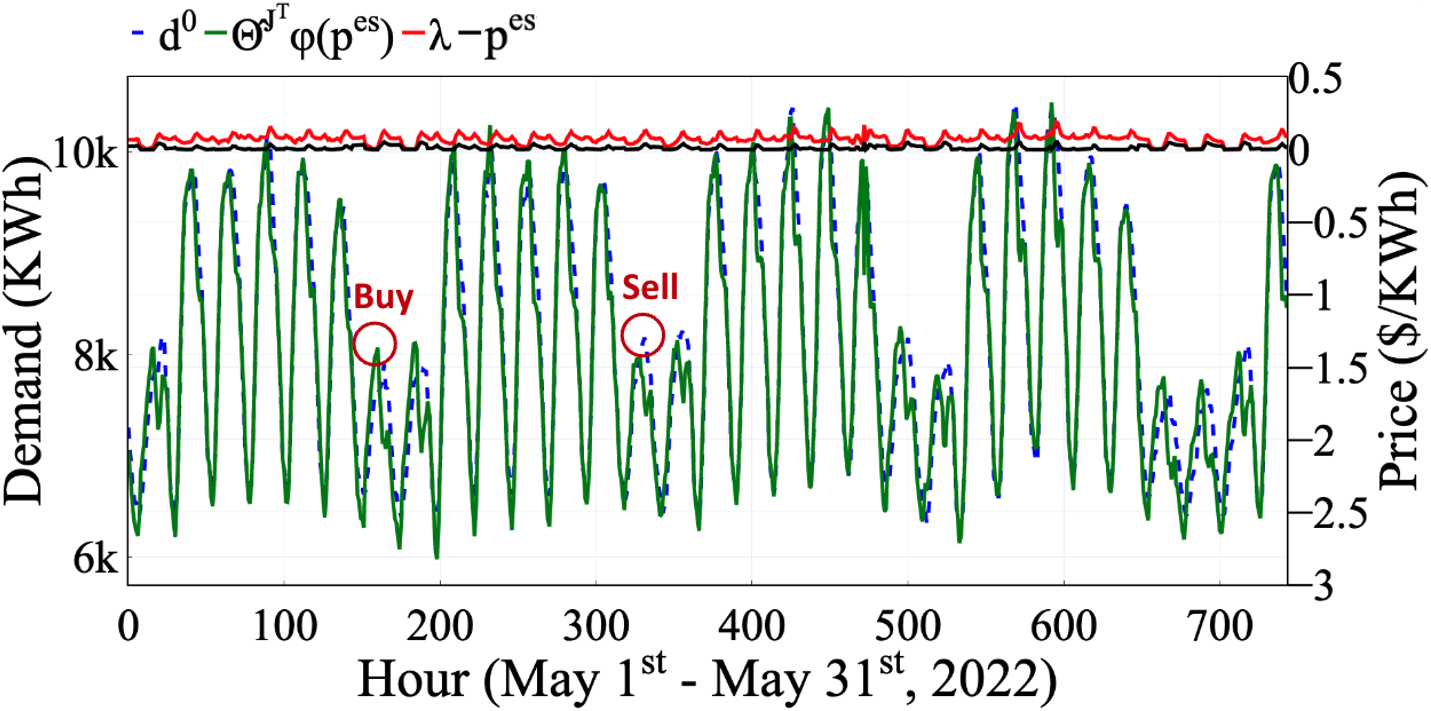} }}
    \caption{Results for $a=2$ and $Q_i=0.01\cdot W_{i}$, for May 1\ts{st}-31\ts{st}, 2022.}
    \label{fig:a2_Q001}
\end{figure*}
\begin{figure*}[!ht]
    \centering
    \subfloat[\centering Daily player utilities for variable number of samples $j$.]{{\label{fig:a2_Q001_samples} \includegraphics[width=7.9cm]{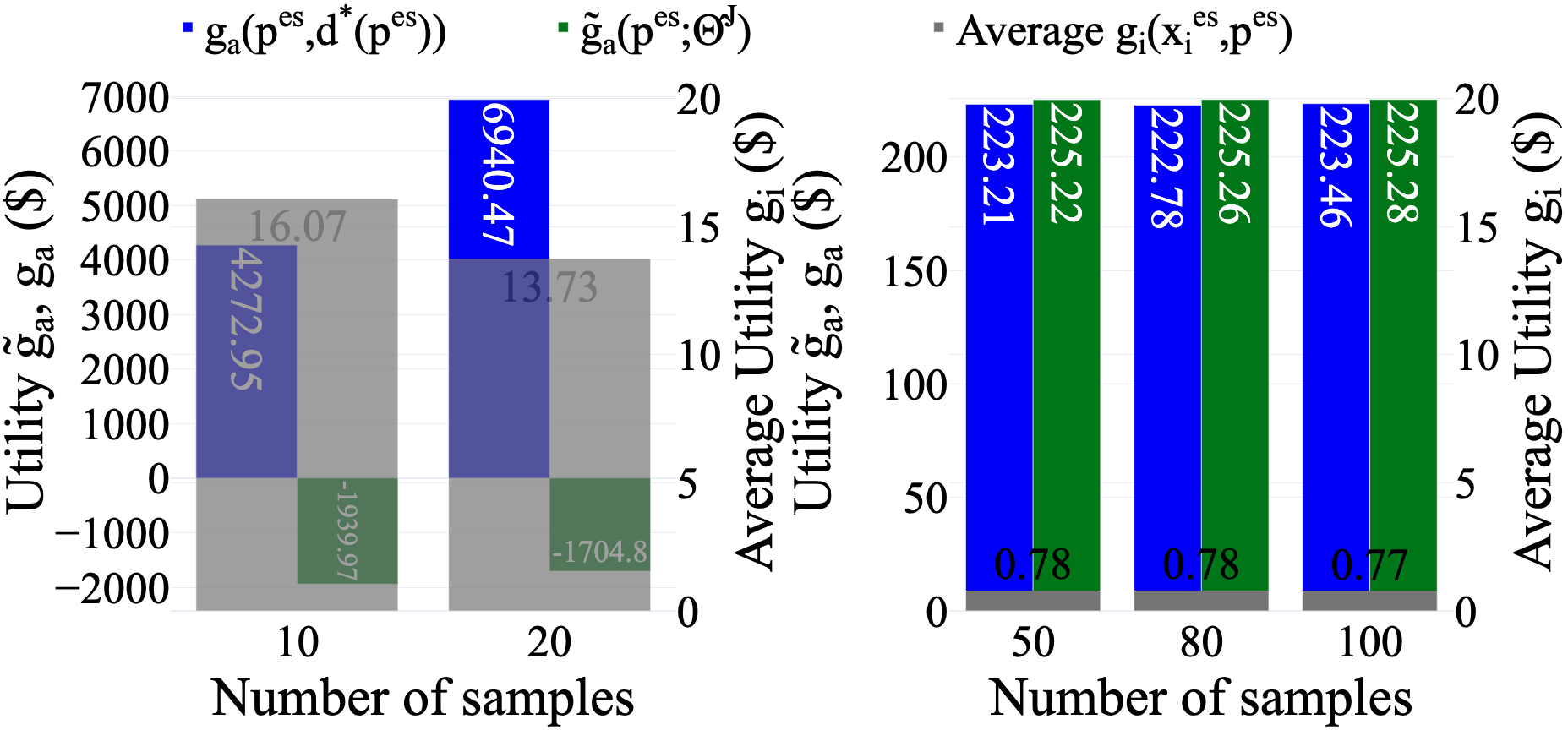}}}
    \qquad
    \subfloat[\centering Daily player utilities for variable number of prosumers $N$.]{{\label{fig:a2_Q001_prosumers}\includegraphics[width=7.9cm]{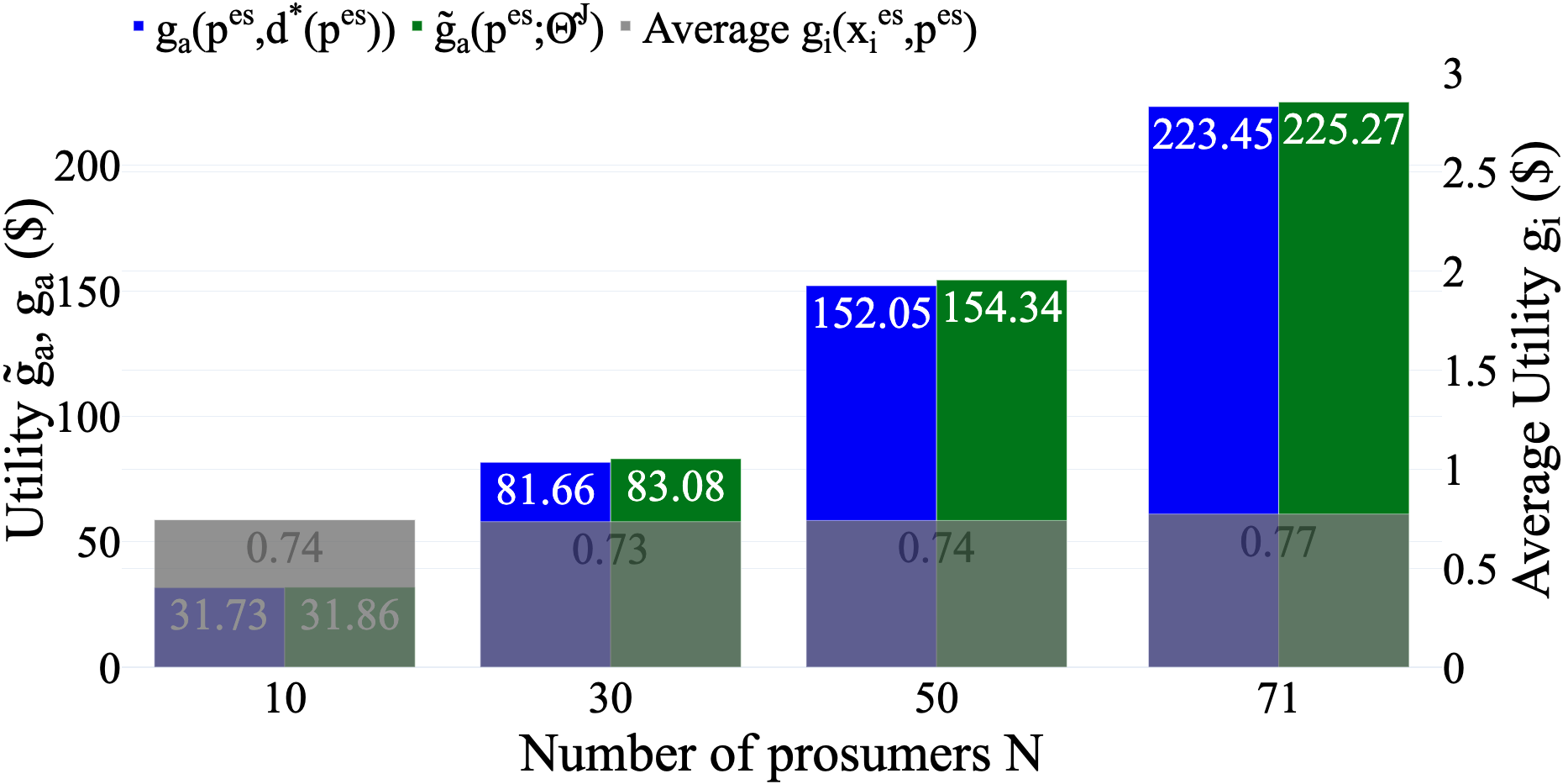} }}
    \caption{Player utilities for $a{=}2$, $Q_i{=}0.01\cdot W_{i}$, averaged over May 2022. Prosumers' utilities are averaged over all prosumers.}
\end{figure*} 

\begin{figure*}[!ht]
    \centering
    \subfloat[\centering Daily player utilities averaged over entire May 2022.]{{\includegraphics[width=8.0cm]{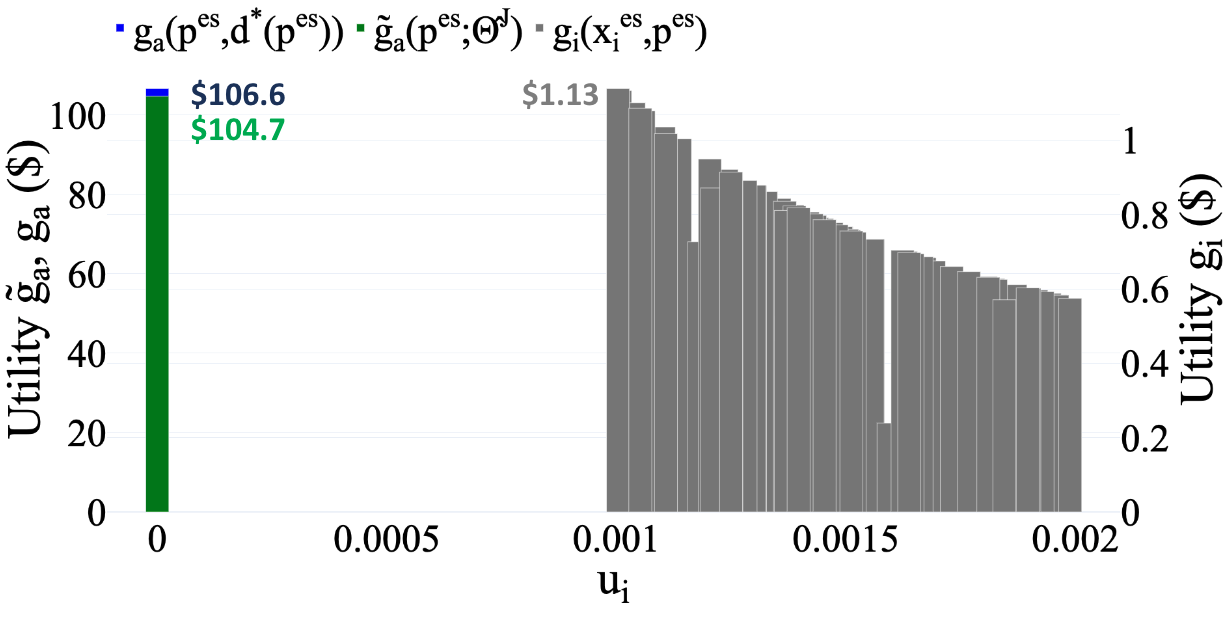} }}
    \qquad
    \subfloat[\centering Prices and demands.]{{\includegraphics[width=8.5cm]{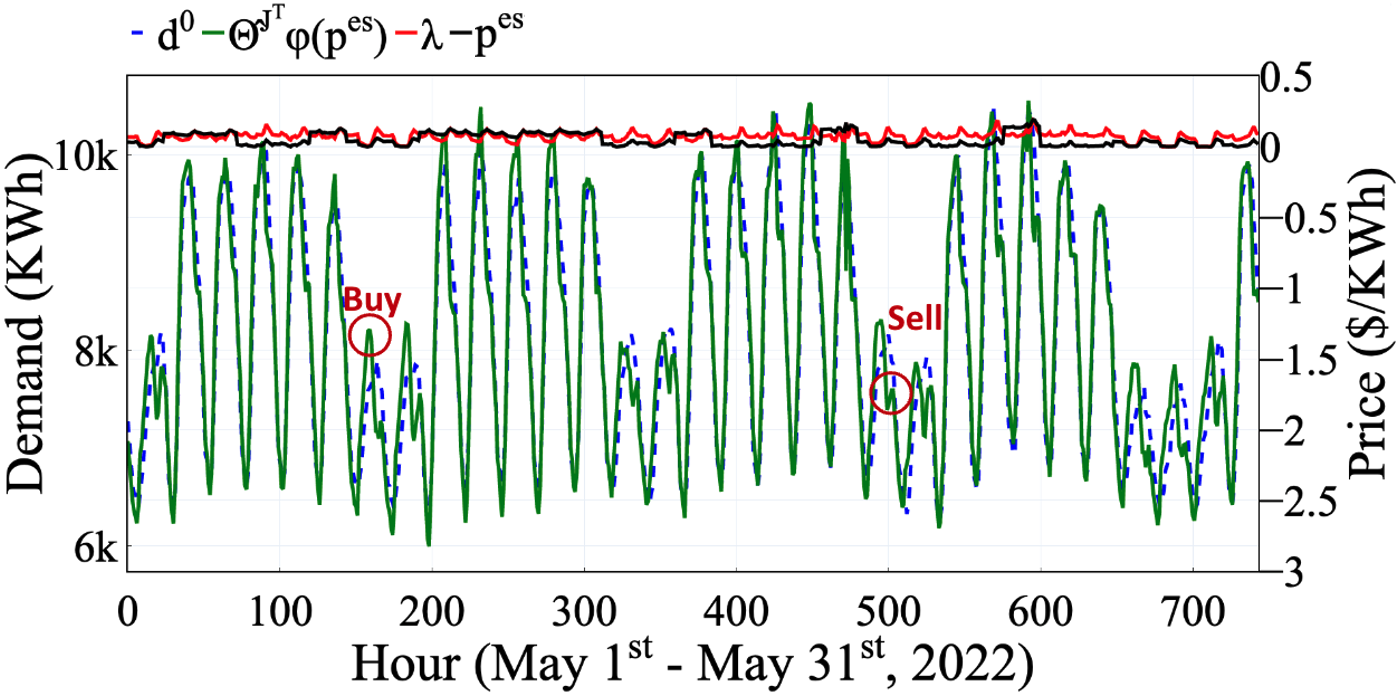} }}
    \caption{Results for $a=2$ and $Q_i=0$, for May 1\ts{st}-31\ts{st}, 2022.}
    \label{fig:a2_Q0}
\end{figure*}    

\begin{figure*}[!ht]
    \centering
    \subfloat[\centering CAISO net demand trend, a.k.a. the duck curve \cite{DOE_duck_curve_CAISO}.]{{\label{fig:duck_curve} \includegraphics[width=7.7cm,height=4.5cm]{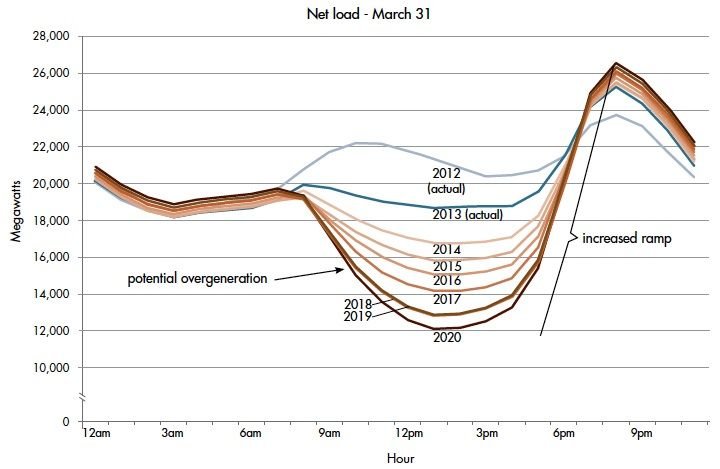} }}
    \qquad
    \subfloat[\centering CAISO net demand before (red) \cite{CAISO_net_demand_data} and after (purple) DR.]{{\label{fig:duck_alleviation}\includegraphics[width=7.7cm,height=4.2cm]{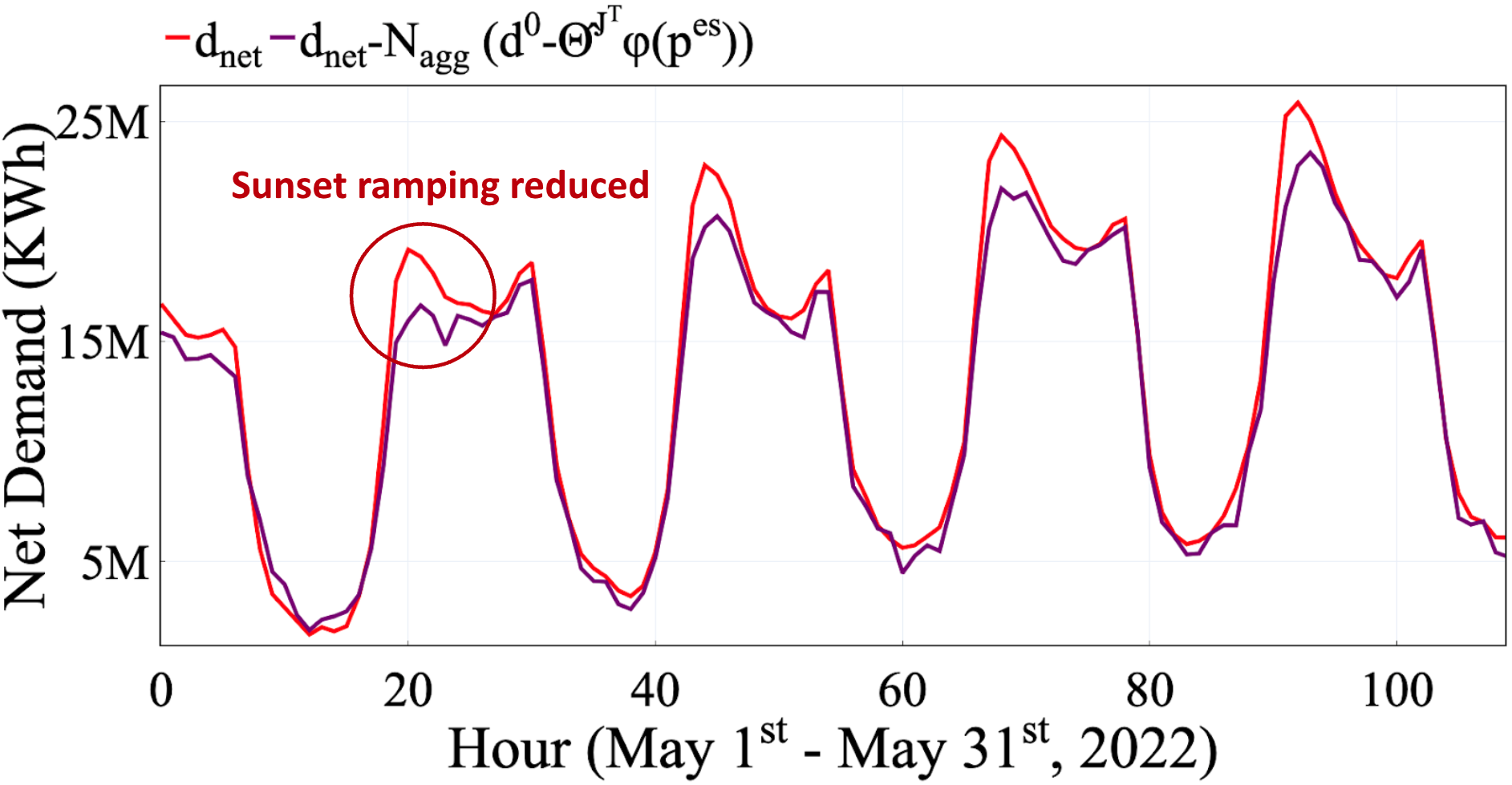} }}
    \caption{Cumulative effect of $N_{agg}{=}1000$ identical DR-aggregators on duck-curve ramping, for $a=2$ and $Q_i=0.1\cdot W_i$.}
\end{figure*}
\begin{figure*}[!ht]
    \centering
    \subfloat[\centering Linear complexity with number of samples $J$.]{{\label{fig:sample_complex} \includegraphics[width=7.7cm,height=4.2cm]{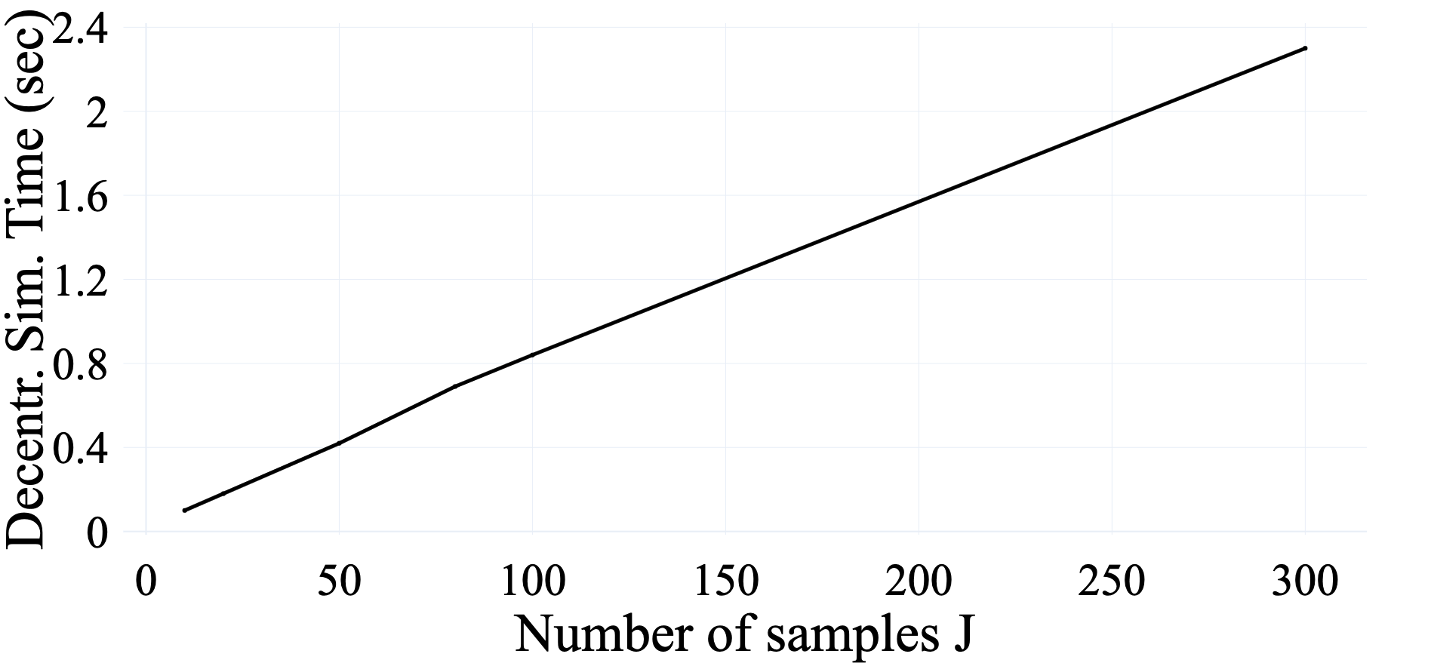}}}
    \qquad
    \subfloat[\centering Linear (at worst) complexity with number of prosumers $N$.]{{\label{fig:pros_complex}\includegraphics[width=7.7cm,height=4.2cm]{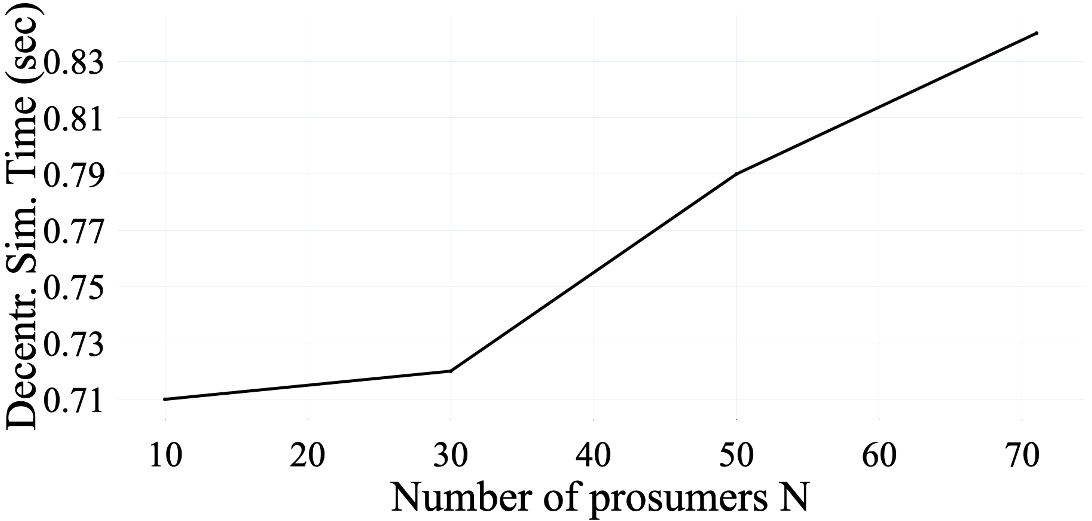} }}
    \caption{Experimental results on complexities of the proposed decentralized Algorithm \ref{algo:stackelberg_algo_online}.}
\end{figure*}

This research suffers from two limitations. First, as $Q_i$ increases, $\mathbf{x}_i(\mathbf{p})$ shifts away from $\mathbf{x}_i^{\o{}}$, to satisfy the constraint (\ref{pros_eq_con}), increasing the prosumer's inconvenience (\ref{pros_obj}). To avoid $g_i {\leq} 0$ for large $Q_i$, the $u_i$ is reduced and could potentially lose its meaning. Alternatives, such as logarithmic inconvenience functions, cannot capture the symmetry in deviations from increasing or decreasing demand. Second, the added competition from more DR-aggregators leads to LMP price drop and diminishing marginal returns \cite{PJM_market_rep_2019} but also to further CO$_2$ emissions reduction due to higher volumes cleared.

\section{Conclusion and Future Work}\label{sec:concl}
In this work, an energy trading Stackelberg game mechanism between the prosumers and the DR-aggregator was designed to engage more distributed resources in the pathway to 2030 decarbonization and 2050 net-zero economy. The difficulties associated with deriving closed-form equilibrium solutions for market bidding motivated the decentralized, privacy-preserving, scalable learning algorithm for finding approximate game equilibria through the decoupling of the players' problems, outperforming less scalable MPEC solutions for Stackelberg games (bilevel programs). Moreover, bounds to the approximate equilibrium solution were derived via the stability to the parameter perturbation. Experimental results utilizing California market prices and demands were presented. 

Future work includes the incorporation of more detailed distributed load constraints in the prosumer's problem \cite{he2013engage, varlamis2022smart}.





\bibliographystyle{IEEEtran}
\bibliography{references}
\section{Appendix}\label{sec:monot_conv}
\subsection{Proof of Theorem \ref{thm:UUB_PE_with_wo_appr_error}}\label{sec:proof_thm_UUB_PE_with_wo_appr_error}
\subsubsection{Part 1: Without Approximation Errors}\label{sec:part1_wo_appr_errors}
Define $\bm{\Phi}^J \in \mathbb{R}^{K \times J}$ the matrix that contains the basis vectors $\bm{\phi}^1,...,\bm{\phi}^J$, where the uppercase exponent $J$ on $\bm{\Phi}^J$ denotes the second dimension of this matrix. Then, equivalently to \cite{ioannou2006adaptive} the following holds,
\begin{equation}
    \bm{\Pi}^{j^{-1}} {-} m \bm{\Pi}^{{j-1}^{-1}} {=} \bm{\Phi}^J \bm{\Phi}^{J^\mathsf{T}}
    {-} \bm{\Phi}^{J-1} \bm{\Phi}^{{J-1}^\mathsf{T}}
    {=} \bm{\phi}^j \bm{\phi}^{j^\mathsf{T}},
    \label{Pi_diff}
\end{equation}
\noindent
where $\bm{\Pi}^j \in \mathbb{R}^{K \times K}$, because element-wise subtraction equals  $\sum_{j=1}^J \phi_{tj}^2 - \sum_{j=1}^{J-1} \phi_{tj}^2 = \phi_{tJ}^2$ and $\sum_{j=1}^J \phi_{tj} \phi_{\tau j} - \sum_{j=1}^{J-1} \phi_{tj} \phi_{\tau j} = \phi_{tJ}\phi_{\tau J}$ for matrix entries $t \neq \tau$.
Assume the case where $\bm{\epsilon}_m(\mathbf{p})=0$. Adding equations (\ref{Pi_diff}) over a window of size of $M$ results in the following statement,
\begin{equation*}
\begin{split}
    & \bm{\Pi}^{{k-1}^{-1}} {+} \bm{\Pi}^{{k}^{-1}} {+} ... {+} \bm{\Pi}^{{k+M-1}^{-1}} \\ 
    & {=} m (\bm{\Pi}^{{k-2}^{-1}} {+} \bm{\Pi}^{{k-1}^{-1}} {+} ... {+} \bm{\Pi}^{{k+M-2}^{-1}})
    {+} \sum_{j=k}^{k+M} \bm{\phi}^j \bm{\phi}^{j^\mathsf{T}},
\end{split}
\end{equation*}
\noindent
where $0 {<}  m {\leq} 1$, $\bm{\Pi}^{{k-2}^{-1}} {\succ} 0$, $\bm{\Pi}^{{k-1}^{-1}} {\succ} 0$, ..., $\bm{\Pi}^{{k+M-2}^{-1}} {\succ} 0$,
\begin{equation*}
    \bm{\Pi}^{{k-1}^{-1}} {+} \bm{\Pi}^{{k}^{-1}} {+} ... {+} \bm{\Pi}^{{k+M-1}^{-1}} \geq \sum_{j=k}^{k+M} \bm{\phi}^j \bm{\phi}^{j^\mathsf{T}} \geq \beta_0 \mathbf{I},
\end{equation*}
\noindent
where the last bound is from the application of PE condition (Theorem \ref{thm:UUB_PE_with_wo_appr_error}). From (\ref{Pi_diff}) it is true that $\bm{\Pi}^{{k-1}^{-1}} {\succ} m \bm{\Pi}^{{k-2}^{-1}}$, since $\bm{\phi}^j \bm{\phi}^{j^\mathsf{T}} \succeq 0$. Therefore,
\begin{equation*}
    (m^{-M} {+} m^{{-}(M{-}1)} {+} ... {+} 1) \bm{\Pi}^{{k+M-1}^{-1}} \geq \beta_0 \mathbf{I},
\end{equation*}
\noindent
and since,
\begin{equation*}
\begin{split}
    m^{-M} & {+} m^{{-}(M{-}1)} {+} ... {+} 1 \\
    & = m^{-M} (1{+} m {+} m^2 {+} ... {+} m^{M-1} {+} m^M) \\
    & = m^{-M} \frac{m^{-M} - 1}{m^{-M} - 1} (1{+} m {+} m^2 {+} ... {+} m^{M-1} {+} m^M) \\
    & = m^{-M} \frac{m^{-1} - m^M}{m^{-1} - 1}
    = \frac{m^{-(M+1)-1}}{m^{-1}-1},
\end{split}
\end{equation*}
\noindent
we get $\forall j \geq M$,
\begin{equation}
    \bm{\Pi}^{{j-1}^{-1}} 
    \succeq \frac{(m^{-1}{-}1)}{m^{-(M+1)}-1} \beta_0 \mathbf{I} \succeq 0.
    \label{Pi_inv_PSD}
\end{equation}
\noindent
From Lemma \ref{lemma:approx_error} and $\bm{\epsilon}_m(\mathbf{p})=0$, we have that $\mathbf{d}^{\star}(\mathbf{p}) = \bm{\Theta_c}^{{\star}^{\mathsf{T}}} \bm{\phi}(\mathbf{p}) = \bm{\Theta}^{{\star}^{\mathsf{T}}} \bm{\phi}(\mathbf{p})$. If we introduce the latter into $\bm{\hat{\Theta}}^{J} = \bm{\hat{\Theta}}^{J-1} + \bm{\Pi}^j \bm{\phi}^j (\mathbf{d}^{j \star} - \bm{\hat{\Theta}}^{{J-1}^\mathsf{T}} \bm{\phi}^j)^\mathsf{T}$, we get,
\begin{equation*}
    \bm{\hat{\Theta}}^{j} = \bm{\hat{\Theta}}^{j-1} + \bm{\Pi}^j \bm{\phi}^j (\bm{\Theta}^{{\star}^{\mathsf{T}}} \bm{\phi}^j - \bm{\hat{\Theta}}^{{j-1}^\mathsf{T}} \bm{\phi}^j)^\mathsf{T},
\end{equation*}
\begin{equation*}
    \bm{\Theta}^{\star} {-} \bm{\hat{\Theta}}^{j} = \bm{\Theta}^{\star} {-} \bm{\hat{\Theta}}^{j-1} {-}
    \bm{\Pi}^j \bm{\phi}^j \bm{\phi}^{j^{\mathsf{T}}} (\bm{\Theta}^{\star} {-} \bm{\hat{\Theta}}^{j-1}),
\end{equation*}
\noindent
and by defining $\mathbf{\tilde \Theta}^j = \mathbf{\Theta^{\star}} - \mathbf{\hat{\Theta}}^j$, we get,
\begin{equation}
    \mathbf{\tilde \Theta}^j = (\mathbb{I} - \bm{\Pi}^j \bm{\phi}^j \bm{\phi}^{j^{\mathsf{T}}}) \mathbf{\tilde \Theta}^{j-1}.
    \label{theta_tilde_dynamics}
\end{equation}
\noindent
We now consider the following Lyapunov function candidate,
\begin{equation}
    V_j = \tr\{\mathbf{\tilde \Theta}^{j^{\mathsf{T}}} \bm{\Pi}^{j^{-1}} \mathbf{\tilde \Theta}^{j}\},
    \label{theta_Lyapunov}
\end{equation}
\noindent
and by utilizing (\ref{theta_tilde_dynamics}) in the following difference, we get,
\begin{equation*}
\begin{split}
    V_j {-} V_{j-1} & = \tr\{\mathbf{\tilde \Theta}^{j^{\mathsf{T}}} \bm{\Pi}^{j^{-1}} \mathbf{\tilde \Theta}^{j} {-} \mathbf{\tilde \Theta}^{{j-1}^{\mathsf{T}}} \bm{\Pi}^{{j-1}^{-1}} \mathbf{\tilde \Theta}^{j-1}\} \\
    & = \tr\{\mathbf{\tilde \Theta}^{{j-1}^{\mathsf{T}}}  
    (\mathbf{A}^j{+}\bm{\phi}^j b_j \bm{\phi}^{j^{\mathsf{T}}}) \mathbf{\tilde \Theta}^{j-1}\},
\end{split}
\end{equation*}
\noindent
where for $\mathbf{A}^j \in \mathbb{R}^{K \times K}$, $\alpha_j \in \mathbb{R}$ and $b_j \in \mathbb{R}$ we have,
\begin{equation*}
    \mathbf{A}^j = (m{-}1) \bm{\Pi}^{{j-1}^{-1}} {-} m \bm{\phi}^j \alpha_j^{-1} \bm{\phi}^{j^{\mathsf{T}}},
\end{equation*}
\begin{equation*}
    \alpha_j = m + \bm{\phi}^{j^{\mathsf{T}}} \bm{\Pi}^{j-1} \bm{\phi}^{j},
\end{equation*}
\begin{equation*}
\begin{split}
    b_j = & 1 {-} \alpha_j^{-1} (m + \bm{\phi}^{j^{\mathsf{T}}} \bm{\Pi}^{j-1} \bm{\phi}^{j}) \\
    & {+} \alpha_j^{-1} ({-}\alpha_j \bm{\phi}^{j^{\mathsf{T}}} \bm{\Pi}^{j-1} \bm{\phi}^{j} {+} \bm{\phi}^{j^{\mathsf{T}}} m  \bm{\Pi}^{j-1} \bm{\phi}^{j} \\
    & {+} \bm{\phi}^{j^{\mathsf{T}}} \bm{\Pi}^{j-1} \bm{\phi}^{j} \bm{\phi}^{j^{\mathsf{T}}} \bm{\Pi}^{j-1} \bm{\phi}^{j}) \alpha_j^{-1} \\
    = & 1 {-} \alpha_j^{-1} \alpha_j {+} \alpha_j^{-1} ({-}\alpha_j \bm{\phi}^{j^{\mathsf{T}}} \bm{\Pi}^{j-1} \bm{\phi}^{j} \\
    & {+} (m {+} \bm{\phi}^{j^{\mathsf{T}}} \bm{\Pi}^{j-1} \bm{\phi}^{j}) \bm{\phi}^{j^{\mathsf{T}}} \bm{\Pi}^{j-1} \bm{\phi}^{j}) \alpha_j^{-1} \\
    = & \alpha_j^{-1} ({-}\alpha_j \bm{\phi}^{j^{\mathsf{T}}} \bm{\Pi}^{j-1} \bm{\phi}^{j} 
    {+} \alpha_j \bm{\phi}^{j^{\mathsf{T}}} \bm{\Pi}^{j-1} \bm{\phi}^{j}) \alpha_j^{-1} = 0.
\end{split}
\end{equation*}
\noindent
It follows that,
\begin{equation*}
\begin{split}
    V_j {-} V_{j-1} = & \tr \{\mathbf{\tilde \Theta}^{{j-1}^{\mathsf{T}}} ((m{-}1)\bm{\Pi}^{{j-1}^{-1}} 
    {-} m \alpha_j^{-1} \bm{\phi}^{j} \bm{\phi}^{j^{\mathsf{T}}}) \mathbf{\tilde \Theta}^{j-1}\} \\
    \leq & \tr \{\mathbf{\tilde \Theta}^{{j-1}^{\mathsf{T}}} (m{-}1)\bm{\Pi}^{{j-1}^{-1}} \mathbf{\tilde \Theta}^{j-1} \}
    = (m{-}1) V_{j-1},
\end{split}
\end{equation*}
\noindent
since $m \alpha_j^{-1} \bm{\phi}^{j} \bm{\phi}^{j^{\mathsf{T}}} \succeq 0$, due to its quadratic form and the fact that $m > 0$ and $\alpha_j^{-1} > 0$. It is then obvious that,
\begin{equation}
    V_j \leq m V_{j-1} \leq m^j V_0 = m^j \tr\{\mathbf{\tilde \Theta}^{0^{\mathsf{T}}} \bm{\Pi}^{0^{-1}} \mathbf{\tilde \Theta}^0\}.
    \label{theta_Lyapunov_order}
\end{equation}
\noindent
Additionally from (\ref{theta_Lyapunov}), we have,
\begin{equation*}
\begin{split}
    V_j = & \tr\{\mathbf{\tilde \Theta}^{j^{\mathsf{T}}} \bm{\Pi}^{j^{-1}} \mathbf{\tilde \Theta}^{j}\} 
    = \tr\{\bm{\Pi}^{j^{-1}} \mathbf{\tilde \Theta}^{j} \mathbf{\tilde \Theta}^{j^{\mathsf{T}}}\} \\
    \geq & \lambda_{min}(\bm{\Pi}^{j^{-1}}) \tr\{ \mathbf{\tilde \Theta}^{j} \mathbf{\tilde \Theta}^{j^{\mathsf{T}}}\} 
    = \lambda_{min}(\bm{\Pi}^{j^{-1}}) \tr\{ \mathbf{\tilde \Theta}^{j^{\mathsf{T}}} \mathbf{\tilde \Theta}^{j}\} \\
    = & \lambda_{min} (\bm{\Pi}^{j^{-1}}) \|\mathbf{\tilde \Theta}^{j} \|_F^2,
\end{split}
\end{equation*}
\noindent
where $\lambda_{min}(\cdot)$ is the minimum eigenvalue which can provide this lower bound for any positive semi-definite matrices $\bm{\Pi}^{j^{-1}}$ and $\mathbf{\tilde \Theta}^{j} \mathbf{\tilde \Theta}^{j^{\mathsf{T}}}$, according to \cite{fang1994inequalities}. Note that $\bm{\Pi}^{j^{-1}} \succeq 0$ from (\ref{Pi_inv_PSD}) and $\mathbf{\tilde \Theta}^{j} \mathbf{\tilde \Theta}^{j^{\mathsf{T}}} \succeq 0$ because for any $\bm{x} \in \mathbb{R}^K$ it is true that $\bm{x}^{\mathsf{T}} (\mathbf{\tilde \Theta}^{j} \mathbf{\tilde \Theta}^{j^{\mathsf{T}}}) \bm{x} = (\mathbf{\tilde \Theta}^{j^{\mathsf{T}}} \bm{x})^{\mathsf{T}} (\mathbf{\tilde \Theta}^{j^{\mathsf{T}}} \bm{x}) = \|\mathbf{\tilde \Theta}^{j^{\mathsf{T}}} \bm{x}\|_2^2 \geq 0$. Then from (\ref{Pi_inv_PSD}), we have $\forall j \geq M$,
\begin{equation*}
    V_j \geq \frac{(m^{-1}{-}1)}{m^{-(M+1)}-1} \beta_0 \| \mathbf{\tilde \Theta}^{j} \|_F^2.
\end{equation*}
\noindent
The last, along with (\ref{theta_Lyapunov_order}) and \cite{fang1994inequalities}, gives,
\begin{equation*}
\begin{split}
    \| & \mathbf{\tilde \Theta}^{j} \|_F^2 \leq \frac{m^j \tr\{\mathbf{\tilde \Theta}^{0^{\mathsf{T}}} \bm{\Pi}^{0^{-1}} \mathbf{\tilde \Theta}^0\} (m^{-(M+1)}-1)}{\beta_0 (m^{-1}-1)} \\
    & \leq \frac{m^j (m^{-(M+1)}-1)}{\beta_0 (m^{-1}-1)} \lambda_\textrm{max}(\bm{\Pi}^{0^{-1}}) \| \mathbf{\tilde \Theta}^{0} \|_F^2 = \zeta m^j \| \mathbf{\tilde \Theta}^{0} \|_F^2.
    \label{expon_conv_wo_errors}
\end{split}
\end{equation*}
\noindent
Since $\lim_{j \rightarrow \infty} \| \mathbf{\tilde \Theta}^{j} \|_F^2 {\leq} \lim_{j \rightarrow \infty} \zeta m^j \| \mathbf{\tilde \Theta}^{0} \|_F^2 {=} 0$ for $0 {<} m {\leq} 1$, then $\mathbf{\hat{\Theta}}^j {\rightarrow} \mathbf{\Theta^{\star}}$. This result was proved before in \cite{ioannou2006adaptive} for vector $\mathbf{\tilde \Theta}^{j}$, not matrix (single-output systems).\frQED

\subsubsection{Part 2: With Approximation Errors}\label{sec:part2_with_appr_errors}
\noindent
Assume that $\bm{\epsilon}_m(\mathbf{p}) \neq 0$. The dynamics (\ref{theta_tilde_dynamics}) now become,
\begin{equation}
    \mathbf{\tilde \Theta}^j = (\mathbb{I} - \bm{\Pi}^j \bm{\phi}^j \bm{\phi}^{j^{\mathsf{T}}}) \mathbf{\tilde \Theta}^{j-1} {-} \bm{\Pi}^j \bm{\phi}^j \bm{\epsilon}_m(\mathbf{p}^j_r)^{\mathsf{T}}.
    \label{theta_dynamics_with_error}
\end{equation}
\noindent
In Part 1 of Theorem \ref{thm:UUB_PE_with_wo_appr_error}, exponential convergence was shown. This result leads directly to, 
\begin{equation*}
    \|\prod_{i=k}^{j-1} (\mathbb{I} - \bm{\Pi}^i \bm{\phi}^i \bm{\phi}^{i^{\mathsf{T}}})\|_F \leq \zeta^{1/2} m^{(j-k)/2},
\end{equation*}
\noindent
$\forall j \geq M$ and $k \geq 0$. Then, (\ref{theta_dynamics_with_error}) results in,
\begin{equation*}
\begin{split}
    \| \mathbf{\tilde \Theta}^{j} & \|_F \leq  \zeta^{\frac{1}{2}} m^{\frac{j}{2}} \|\mathbf{\tilde \Theta}^0 \|_F 
    {+} \sum_{k=0}^{j-1} \zeta^{\frac{1}{2}} m^{\frac{j-k-1}{2}} \|\bm{\Pi}^k \bm{\phi}^k \bm{\epsilon}_m(\mathbf{p}^k_r)^{\mathsf{T}} \|_F \\
    \leq & \zeta^{\frac{1}{2}} m^{\frac{j}{2}} \|\mathbf{\tilde \Theta}^0 \|_F {+} \sum_{k=0}^{j-1} \zeta^{\frac{1}{2}} m^{\frac{j-k-1}{2}} \|\bm{\Pi}^k \bm{\phi}^k \|_2 \| \bm{\epsilon}_m(\mathbf{p}^k_r)^{\mathsf{T}} \|_2 \\ 
    \leq & \zeta^{\frac{1}{2}} m^{\frac{j}{2}} \|\mathbf{\tilde \Theta}^0 \|_2 
    {+} \zeta^{\frac{1}{2}} \sum_{k=0}^{j-1} m^{\frac{j-k-1}{2}} \|\bm{\Pi}^k \bm{\phi}^k \|_2 \epsilon_m^\mathrm{max}\\
    \leq & \zeta^{\frac{1}{2}} m^{\frac{j}{2}} \|\mathbf{\tilde \Theta}^0 \|_F {+} \zeta^{\frac{1}{2}} m^{\frac{j}{2}} \zeta_1 \epsilon_m^\mathrm{max} 
    {+} \zeta^{\frac{1}{2}} \zeta_2 \frac{m^{\frac{j}{2}}{-}1}{m^{\frac{1}{2}}{-}1} \epsilon_m^\mathrm{max}{,}
\end{split}
\end{equation*}
where,
\begin{equation*}
    \zeta_1 = m^{-1/2} \frac{m^{-M/2}-1}{m^{-1/2}-1} \sum_{k=0}^{M-1} \|\bm{\Pi}^k \bm{\phi}^k \|_2,
\end{equation*}
\begin{equation*}
    \zeta_2 = \frac{m^{-(M+1)}-1}{\beta_0 (m^{-1}-1)} \beta_1^{1/2}.
\end{equation*}
\noindent
The last term which includes $\zeta_2$, is proved based on $\|\bm{\Pi}^k \bm{\phi}^k \|_2 \leq \zeta_2$, which is a direct result of $\|\bm{\Pi}^{k}\|_F \leq \frac{m^{-(M+1)}-1}{\beta_0(m^{-1}-1)}$ by (\ref{Pi_inv_PSD}) and on the definition of PE $\|\bm{\phi}^k\|_2 \leq \lambda_\textrm{max}^{1/2}(\bm{\phi}^{k^{\mathsf{T}}} \bm{\phi}^k) \leq \beta_1^{1/2}$.  Now, if we take $j \rightarrow \infty$,
\begin{equation*}
    \lim_{j \rightarrow \infty} \|\mathbf{\tilde \Theta}^{j} \|_F \leq \zeta^{1/2} \frac{\zeta_2}{1-m^{1/2}} \epsilon_m^\mathrm{max} = \eta \epsilon_m^\mathrm{max}.
\end{equation*}
\noindent
Since $\zeta^{1/2} \frac{\zeta_2}{1-m^{1/2}} \epsilon_m^\mathrm{max} \geq 0$, we conclude that $\mathbf{\tilde \Theta}^j$ converges exponentially to a bounded ball that includes $\mathbf{\Theta}^{\star}$. These results agree with \cite{yang2019adaptive, vamvoudakis2010online}, where $\mathbf{\tilde \Theta}^{j}$ is a vector instead of a matrix (single-output systems). One can show, that the difference $V_j {-} V_{j-1} {<} 0$, once $\mathbf{\tilde \Theta}^j$ enters a specific set in finite time $j{<}\infty$. Then the Lyapunov Extension Theorem \cite{lewis2020neural} demonstrates that $\mathbf{\tilde \Theta}^j$ from (\ref{theta_dynamics_with_error}) and the corresponding $\bm{\hat{\Theta}}^j$, are Uniformly Ultimately Bounded (UUB). This proof is omitted for brevity, but a similar can be found in \cite{vamvoudakis2010online}.\frQED

\end{document}